\documentclass{article}

\usepackage[preprint]{neurips_2026}

\usepackage[utf8]{inputenc} 
\usepackage[T1]{fontenc}    
\usepackage{hyperref}       
\usepackage{url}            
\usepackage{booktabs}       
\usepackage{amsfonts}       
\usepackage{nicefrac}       
\usepackage{microtype}      
\usepackage{xcolor}         

\usepackage{amsmath}
\usepackage{amsthm}
\usepackage{graphicx}
\usepackage{subcaption}

\newtheorem{theorem}{Theorem}[section]
\newtheorem{corollary}[theorem]{Corollary}
\newtheorem{lemma}[theorem]{Lemma}
\newtheorem{assumption}[theorem]{Assumption}
\newtheorem{proposition}[theorem]{Proposition}
\newtheorem{remark}[theorem]{Remark}

\newcommand{\R}{\mathbb{R}}
\newcommand{\E}{\mathbb{E}}
\newcommand{\Prob}{\mathbb{P}}
\newcommand{\norm}[1]{\left\lVert #1\right\rVert}
\DeclareMathOperator{\Cov}{Cov}
\DeclareMathOperator{\Var}{Var}
\DeclareMathOperator{\tr}{tr}

\title{Self-Attention as a Covariance Readout: A Unified View of In-Context Learning and Repetition}

\author{
  Haoren Xu \\
  Fudan University \\
  \texttt{haorenxu25@m.fudan.edu.cn} \\
  \And
  Guanhua Fang \\
  Fudan University \\
  \texttt{fanggh@fudan.edu.cn} \\
}

\begin{document}

\maketitle

\begin{abstract}
Large language models (LLMs) exhibit two striking and ostensibly unrelated behaviours: in-context learning (ICL) and repetitive generation. In both, the model behaves as though it had summarised the context into a population-level statistic and discarded token-level detail. We ask whether this ``summarisation and forgetting'' can be derived from the attention mechanism itself, and answer in the affirmative. Under stationary, ergodic and elliptical inputs, the softmax attention output converges almost surely to $\Theta_V\Sigma\Theta_K^{\top}\Theta_Q x_t$, where $\Sigma$ is the input covariance; the long-context limit is therefore a linear readout of the input's second-order statistics. Two consequences follow. (i) For in-context linear regression, a single softmax head can implement one step of population gradient descent. Stacking such heads with residual connections iterates this update and implements multiple gradient descent steps. (ii) Propagated across an $L$-layer transformer, this readout drives the terminal hidden state at the parametric $1/t$ rate to a deterministic function of the current token alone, so that autoregressive generation collapses asymptotically to a first-order Markov chain whose attracting orbits furnish a structural account of repetition and mode collapse. The two phenomena thus emerge as facets of a single covariance-readout principle.
\end{abstract}

\section{Introduction}
\label{sec:intro}

Modern large language models built on transformers \citep{vaswani2017attention} display behaviors that the architecture itself does not obviously predict. In-context learning (ICL) and repetitive generation have attracted particular attention.

When a pre-trained model is given a few input--output pairs in its prompt followed by a query input, it frequently produces outputs consistent with the demonstrated mapping, without any parameter update \citep{brown2020language}. This raises the question of how the model extracts the mapping. Existing explanations include implicit gradient descent on the demonstrations \citep{von2023transformers,akyurek2022learning,dai2023can}; implicit Bayesian inference over a latent task variable \citep{xie2021explanation,wang2023large}. A common thread across these interpretations is that the model reads part of its context into an internal regime---a latent distribution, rule, or task descriptor---that governs downstream predictions, while the specific identities of the demonstrations become secondary.

When the same models are sampled autoregressively, they are prone to fall into loops, producing near-verbatim repetitions of recent phrases, stylistic motifs, or topics \citep{holtzman2019curious,welleck2019neural,fu2021theoreticalanalysisrepetitionproblem,xu2022learning}. The model appears to settle into a persistent theme and fail to escape it; aggressive decoding schemes such as nucleus sampling \citep{holtzman2019curious} or objective modifications such as unlikelihood training \citep{welleck2019neural} mitigate but do not eliminate the effect. In the language of the preceding paragraph, the model behaves as if it has again inferred a regime---here, a persistent theme---and is unable to break out of it.

Though surface manifestations differ, both phenomena share a structural signature: the model extracts a regime from its context and then largely forgets which specific tokens induced that regime. In ICL, the regime is the latent input--output mapping implied by the demonstrations; the precise identities of the demonstrations are unimportant provided they pin down the mapping. In repetitive generation, the regime is a persistent theme or motif. In both cases, the model behaves as if it had summarized the context into a population-level statistic and discarded token-level detail. The central question of this paper is whether this ``summarization and forgetting'' can be derived from the attention mechanism itself.

We show that this summarization and forgetting is precisely what a softmax attention head performs. Under stationary, ergodic assumptions and weak-dependence conditions on the input process, the attention output converges to an exponentially tilted population mean that depends on the input only through its covariance. In the long-context limit, an attention head therefore acts as a covariance readout: a linear functional of the current token $x_t$ whose matrix is determined jointly by the attention parameters $(\Theta_Q,\Theta_K,\Theta_V)$ and the input covariance $\Sigma$. Both ICL and repetitive generation follow from this single principle. For ICL, we exhibit a parameter choice, under which a single softmax head can implement one step of population gradient descent. Stacking such heads with residual connections iterates this update and implement multiple gradient descent steps. For repetition, we show that propagating this readout through $L$ stacked layers produces, in the long-context limit, a deterministic position-wise map of the current token; autoregressive generation from such a composite is asymptotically a first-order Markov chain, whose attracting orbits provide a structural account of repetition and mode collapse.

Our contributions are as follows:
\begin{enumerate}
\item We derive closed-form directional limits under rotational and elliptical inputs, exposing the joint role of $(\Theta_Q,\Theta_K,\Theta_V)$ and $\Sigma$ as a covariance readout. We prove a parametric $O(\tau_\alpha\,\tr(\Sigma)/n)$ $L^2$ concentration rate under sub-Gaussian strongly mixing inputs, quantifying $\tr(\Sigma)$ as the effective cost of dimension.
\item We specialize the covariance-readout principle to in-context linear regression. A single head with distribution-determined projections realizes one step of population gradient descent, while a residual stack of such heads iterates this update across depth and progressively approaches the Bayes-optimal linear predictor.
\item We propagate the single-head result through depth and autoregressive decoding, obtaining a first-order Markov-chain limit for generation and connecting its attractor structure to repetition phenomena.
\end{enumerate}

\section{Related work}
\label{sec:related}

\paragraph{Theoretical analyses of ICL.}
A first strand of work treats ICL as learned optimization inside the forward pass. \citet{von2023transformers} and \citet{akyurek2022learning} construct explicit transformer weights that implement one or more steps of gradient descent on the demonstration loss; \citet{dai2023can} frame the phenomenon as meta-optimization; \citet{zhang2024trained} study the trained solutions rigorously in a linear setting. A second strand interprets ICL as implicit Bayesian inference over a latent task variable, with the prompt acting as evidence \citep{xie2021explanation,wang2023large}. A third strand emphasizes mechanistic circuits: \citet{olsson2022context} identify induction heads as a token-level pattern-matching substrate, while \citet{elhage2021mathematical} develop a compositional account of attention-based computation. Data-distributional accounts \citep{chan2022data} situate ICL as an emergent property of training on distributions with burstiness and long-tailed structure. Our analysis occupies a complementary regime: we treat the attention head as a statistical estimator operating on a stationary ergodic input stream and exhibit a parameterization that implements gradient descent.

\paragraph{Repetition and degeneration.}
Repetition is a long-standing failure mode of neural text generation. \citet{holtzman2019curious} analyze likelihood-maximizing decoding and propose nucleus sampling; \citet{welleck2019neural} introduce unlikelihood training to explicitly penalize recently emitted tokens. \citet{fu2021theoreticalanalysisrepetitionproblem} provide a token-level explanation through a high-inflow phenomenon in the induced transition graph, and \citet{xu2022learning} analyze repetition loops from a self-reinforcement perspective. These works operate at the decoding or token-statistics level and largely treat the underlying model as a black-box conditional distribution. Our Markov-closure result is complementary: we show that the conditional distribution itself collapses, in the long-context limit, to a first-order kernel determined by a depth-$L$ composition of covariance readouts, so that repetition and mode collapse become structural properties of the ergodic transformer, not artifacts of any particular sampler. This also reframes the high-inflow diagnosis of \citet{fu2021theoreticalanalysisrepetitionproblem}: tokens with many incoming transitions in the generated graph are those lying in the basin of attraction of the depth-$L$ context-free map.

\section{Theoretical Framework}

\subsection{Attention as a softmax barycenter}
\label{sec:setup}

\paragraph{Notation.}
For nonzero $u,v$ in a common Euclidean space, $\cos(u,v):=u^\top v/(\norm{u}\,\norm{v})$. Vector norms are Euclidean and matrix norms are the corresponding operator norms.

\paragraph{Attention layer.}
Fix integers $d,d_k,d_v\ge 1$ for the input, key/query, and value dimensions. A causal single-head attention layer maps $(x_j)_{j\ge 1}\subset\R^d$ to outputs via parameters $\Theta_Q,\Theta_K\in\R^{d_k\times d}$ and $\Theta_V\in\R^{d_v\times d}$:
\[
q_t=\Theta_Q x_t,\quad k_j=\Theta_K x_j,\quad v_j=\Theta_V x_j,\quad
y_t=\sum_{j=1}^t a_{tj}v_j,\quad
a_{tj}=\frac{e^{s_{tj}}}{\sum_{i\le t}e^{s_{ti}}},\quad s_{tj}=\frac{q_t^\top k_j}{\sqrt{d_k}}.
\]

\paragraph{Barycenter factorization.}
Because $v_j=\Theta_V x_j$ is linear in $x_j$, the projection commutes with softmax weighting, giving
\begin{equation}\label{eq:barycenter}
y_t=\Theta_V\,\bar X_t(q_t),\qquad
\bar X_t(q):=\frac{\sum_{j\le t}x_j\,e^{q^\top\Theta_K x_j/\sqrt{d_k}}}{\sum_{i\le t}e^{q^\top\Theta_K x_i/\sqrt{d_k}}}.
\end{equation}
The attention output is thus the value projection of an exponentially tilted empirical mean of the inputs, with tilt governed by $q^\top\Theta_K x_j/\sqrt{d_k}$. Analysis of $y_t$ reduces to analysis of the conditional mean $\bar X_t(q)$; once the direction of this mean is known, multiplication by $\Theta_V$ delivers the attention output.

\paragraph{Probabilistic framework.}
We model $(x_j)_{j\ge 1}$ as a random process on $\R^d$. For $q\ne 0$, set $p:=\Theta_K^\top q$ and define
\[
\bar X_n(q):=\frac{\sum_{j=1}^n x_j e^{p^\top x_j/\sqrt{d_k}}}
{\sum_{i=1}^n e^{p^\top x_i/\sqrt{d_k}}},\qquad
\Lambda(u):=\log\E\,e^{u^\top x_1}\text{ (wherever finite)}.
\]
The identification of the ergodic limit depends only on the marginal law of $x_1$; rate statements depend additionally on the dependence structure.

The following elementary lemma, used repeatedly to convert $L^2$ errors into directional (cosine) errors, is proved in Appendix~\ref{app:elementary}.

\begin{lemma}[Geometric stability]\label{lem:stability}
Let $u\ne 0$ and $\norm e\le\eta\,\norm u$ with $\eta\in[0,1)$. Then $u+e\ne 0$ and $\cos(u+e,u)\ge\sqrt{1-\eta^2}$.
\end{lemma}

\subsection{Ergodic alignment of the attention output}
\label{sec:ergodic}

We now identify the long-context limit of $y_t$. The argument proceeds in three steps: an unconditional ergodic law, exact directional alignment under rotational input symmetry, and a joint-covariance readout under elliptical inputs.

\subsubsection{Ergodic law of large numbers}

\begin{theorem}[Ergodic LLN]\label{thm:lln}
Let $(x_j)_{j\ge 1}$ be stationary and ergodic on $\R^d$, and fix $q\ne 0$ with $p:=\Theta_K^\top q$. Assume $\E\,e^{p^\top x_1/\sqrt{d_k}}<\infty$ and $\E[\norm{x_1}e^{p^\top x_1/\sqrt{d_k}}]<\infty$. Then almost surely
\[
\bar X_n(q)\xrightarrow{\textup{a.s.}}m_x(q):=\frac{\E[x_1 e^{p^\top x_1/\sqrt{d_k}}]}{\E[e^{p^\top x_1/\sqrt{d_k}}]}.
\]
If moreover $\Lambda$ is finite on a neighborhood of $p/\sqrt{d_k}$, then $m_x(q)=\nabla\Lambda(p/\sqrt{d_k})$.
\end{theorem}

\begin{proof}
Birkhoff's ergodic theorem, applied separately to the stationary ergodic, integrable sequences $(x_j\,e^{p^\top x_j/\sqrt{d_k}})_{j\ge 1}$ and $(e^{p^\top x_j/\sqrt{d_k}})_{j\ge 1}$, yields almost-sure convergence of the respective Cesàro averages to their expectations, the second of which is strictly positive. Taking the ratio gives the first claim.
Differentiation under the integral sign, valid on the interior of the finiteness region of $\Lambda$, gives $\nabla\Lambda(u)=\E[x_1 e^{u^\top x_1}]/\E[e^{u^\top x_1}]$; setting $u=p/\sqrt{d_k}$ yields the second.
\end{proof}

Since $y_t=\Theta_V\bar X_t(q_t)$ by~\eqref{eq:barycenter}, Theorem~\ref{thm:lln} lifts verbatim: $y_t\to\Theta_V m_x(q_t)$ almost surely. The remainder of this section is about identifying the direction of $m_x(q)$.

\subsubsection{Exact alignment under rotational invariance}
\label{ssec:rotational}

The simplest identifying structure is rotational invariance of the marginal input law.

\begin{theorem}[Rotational-invariance alignment]
\label{thm:rotational}
Assume the hypotheses of Theorem~\ref{thm:lln}, $x_1\not\equiv 0$ almost surely, and $x_1\stackrel{d}{=}Rx_1$ for every $R\in O(d)$. Then for every $q\ne 0$ with $p:=\Theta_K^\top q\ne 0$,
\[
m_x(q)=\gamma(\norm p)\,\Theta_K^\top q,\qquad \gamma(\norm p)>0,
\]
and consequently $\cos(y_t,\Theta_V\Theta_K^\top\Theta_Q x_t)\xrightarrow{\textup{a.s.}}1$ whenever $\Theta_V\Theta_K^\top\Theta_Q x_t\ne 0$.
\end{theorem}

\begin{proof}
Define $T(p):=\E[x_1 e^{p^\top x_1/\sqrt{d_k}}]/\E[e^{p^\top x_1/\sqrt{d_k}}]$, so $m_x(q)=T(\Theta_K^\top q)$.

\emph{Direction.} For every orthogonal $R\in O(d)$, rotational invariance gives, via the change of variable $x\mapsto Rx$,
\[
T(p)=\E[Rx_1\,e^{p^\top Rx_1/\sqrt{d_k}}]/
\E[e^{p^\top Rx_1/\sqrt{d_k}}]=R\,T(R^\top p),
\]
so $T$ is $O(d)$-equivariant: $T(Rp)=R\,T(p)$. Taking $R$ to fix $p$ yields $T(p)=RT(p)$; hence $T(p)\in\mathrm{span}\{p\}$, say $T(p)=\gamma_p p$. Equivariance then forces $\gamma_p$ to depend only on $\norm p$.

\emph{Positivity.} Taking $R=-I$ yields $x_1\stackrel{d}{=}-x_1$, so $Y:=p^\top x_1$ is symmetric. Then
\[
p^\top T(p)=\frac{\E[Y\,e^{Y/\sqrt{d_k}}]}{\E[e^{Y/\sqrt{d_k}}]}
=\frac{\E[Y\sinh(Y/\sqrt{d_k})]}{\E[e^{Y/\sqrt{d_k}}]},
\]
and $y\sinh(y/\sqrt{d_k})\ge 0$ with equality only at $y=0$. Rotational invariance and $x_1\not\equiv 0$ force the conditional law of $x_1/\norm{x_1}$ given $\{x_1\ne 0\}$ to be uniform on $S^{d-1}$, so $Y$ is not almost surely zero (since $p\ne 0$). Therefore $p^\top T(p)>0$, i.e.\ $\gamma(\norm p)>0$.

Finally $y_t\to\Theta_V m_x(q_t)=\gamma(\norm{p_t}) \Theta_V\Theta_K^\top\Theta_Q x_t$, and the positive scalar $\gamma$ does not affect the cosine.
\end{proof}

\subsubsection{Joint-covariance limit under elliptical inputs}
\label{ssec:elliptical}

Relaxing rotational symmetry to ellipticity produces a covariance-rotated alignment.

\begin{theorem}[Joint-covariance limit]
\label{thm:elliptical}
Assume the hypotheses of Theorem~\ref{thm:lln}, and suppose $x_1\stackrel{d}{=}BZ$ with $B\in\R^{d\times d}$ deterministic and invertible, $Z\not\equiv 0$ rotationally invariant, and $\E e^{u^\top Z}<\infty$ on a neighborhood of $u=B^\top\Theta_K^\top q/\sqrt{d_k}$. Let $\Sigma:=BB^\top\succ 0$. Then
\[
m_x(q)=\gamma_\Sigma(q)\,\Sigma\,\Theta_K^\top q,\qquad\gamma_\Sigma(q)>0,
\]
where $\gamma_\Sigma(q)$ depends on $B$ only through $\Sigma$ and the quadratic form $q^\top\Theta_K\Sigma\Theta_K^\top q$. Consequently $y_t\xrightarrow{\textup{a.s.}}\gamma_\Sigma(q_t)\,\Theta_V\Sigma\Theta_K^\top\Theta_Q x_t$, and the cosine with this target tends to $1$ almost surely whenever the target is nonzero.
\end{theorem}

\begin{proof}
Set $p:=\Theta_K^\top q$ and $w:=B^\top p$. Since $p^\top x_1\stackrel{d}{=}p^\top BZ=w^\top Z$,
\[
m_x(q)=B\,\frac{\E[Z\,e^{w^\top Z/\sqrt{d_k}}]}
{\E[e^{w^\top Z/\sqrt{d_k}}]}.
\]
Theorem~\ref{thm:rotational} applied to the rotationally invariant $Z$ with ``query'' $w$ gives the inner ratio as $\tilde\gamma(\norm w)\,w$ with $\tilde\gamma(\norm w)>0$. Thus
\[
m_x(q)=\tilde\gamma(\norm w)\,BB^\top p
=\tilde\gamma(\norm w)\,\Sigma\,\Theta_K^\top q,\qquad
\gamma_\Sigma(q):=\tilde\gamma(\norm{B^\top\Theta_K^\top q})>0.
\]
Since $\norm{B^\top\Theta_K^\top q}^2=q^\top\Theta_K\Sigma\Theta_K^\top q$ depends on $B$ only through$\Sigma$, so does $\gamma_\Sigma(q)$. Left-multiplying by $\Theta_V$ and using $y_t=\Theta_V\bar X_t(q_t)\to\Theta_V m_x(q_t)$ concludes.
\end{proof}

Theorem~\ref{thm:elliptical} is the central structural result. In the long-context limit, the attention head is a linear function of $x_t$,
\[
y_t\;\approx\;\gamma_\Sigma(q_t)\,\underbrace{\Theta_V\Sigma\Theta_K^\top\Theta_Q}_{\text{covariance readout}}\,x_t,
\]
whose matrix couples the attention parameters with the input second-order statistics. The realized context $(x_1,\ldots,x_{t-1})$ enters only through the convergence: in the limit, the head forgets the identities of past tokens and retains only the population covariance $\Sigma$. Empirical results are give in Appendix~\ref{app:exp-sec323}.

\subsection{Finite-sample concentration under strong mixing}
\label{sec:rate}

We now quantify the rate at which the covariance readout is attained. The scaling exposes $\tr(\Sigma)$ as the effective cost of dimension and makes explicit the role of weak dependence through a single scalar.

\begin{assumption}[Elliptical sub-Gaussian inputs with strong mixing]
\label{ass:elliptical}
$(x_j)_{j\ge 1}$ is strictly stationary and ergodic on $\R^d$:
\begin{itemize}
\item[\textup{(a)}] (Elliptical sub-Gaussian marginal) $x_1\stackrel{d}{=}BZ$ with $B\in\R^{d\times d}$ deterministic and invertible, $\Sigma:=BB^\top\succ 0$, $Z$ rotationally invariant on $\R^d$ with $\E Z=0$, $\Cov(Z)=I_d$, and $\E e^{u^\top Z}\le e^{\sigma^2\norm u^2/2}$ for all $u\in\R^d$ and some $\sigma\ge 1$.
\item[\textup{(b)}] (Strong mixing) The Rosenblatt $\alpha$-mixing coefficients satisfy $\tau_\alpha:=1+2\sum_{k\ge 1}\sqrt{\alpha(k)}<\infty$.
\end{itemize}
\end{assumption}

Setting $\alpha\equiv 0$ recovers i.i.d.\ inputs with $\tau_\alpha=1$; condition~(b) is satisfied by geometrically ergodic Markov chains, hidden Markov models, and smooth causal functionals of such processes---the natural structural class for transformer hidden states.

\begin{theorem}[Concentration of the covariance readout]
\label{thm:rate}
Under Assumption~\ref{ass:elliptical}, for every $q\in\R^{d_k}$ with $q^\top\Theta_K\Sigma\Theta_K^\top q\le d_k$ and every $n\ge 1$,
\begin{equation}\label{eq:rate-L2}
\E\norm{\bar X_n(q)-m_x(q)}^{2}\le C_1\,\tau_\alpha\,\frac{\tr(\Sigma)}{n},
\end{equation}
for a constant $C_1$.
\end{theorem}

The proof is given in Appendix~\ref{app:rate}.

Combining Theorem~\ref{thm:rate} with Theorem~\ref{thm:elliptical} and Lemma~\ref{lem:stability} yields a finite-sample alignment bound.

\begin{corollary}[Finite-sample alignment]
\label{cor:rate-alignment}
Under Assumption~\ref{ass:elliptical}, for any $(\Theta_Q,\Theta_K,\Theta_V)$ and $x_t$ with $q_t:=\Theta_Q x_t$ satisfying $q_t^\top\Theta_K\Sigma\Theta_K^\top q_t\le d_k$ and $\Theta_V\Sigma\Theta_K^\top q_t\ne 0$,
\[
\E\!\left[1-\cos\!\bigl(y_t,\,\Theta_V\Sigma\Theta_K^\top\Theta_Q x_t\bigr)\right]
\le
\frac{C\,\tau_\alpha\,\norm{\Theta_V}_{\mathrm{op}}^{2}\,\tr(\Sigma)}
{n\,\gamma_\Sigma(q_t)^{2}\,\norm{\Theta_V\Sigma\Theta_K^\top q_t}^{2}},
\]
where $C$ depends only on $\sigma$ and $\gamma_\Sigma(q_t)$ is the positive scalar of Theorem~\ref{thm:elliptical}.
\end{corollary}

\section{In-context learning as a joint-covariance readout}
\label{sec:icl}

Theorem~\ref{thm:elliptical} established that, under stationary ergodic elliptical inputs, the attention output $y_t$ converges in direction to $\Theta_V\Sigma\Theta_K^\top\Theta_Q x_t$. We now specialise to the token structure used in in-context learning (ICL), in which each input bundles a covariate and a target. The specialisation yields two observations. First, the simplest choice of projections realises, in the ergodic limit, a single step of population gradient descent on the in-context square loss: one self-attention layer is one GD step. Second, stacking such heads with residual connections iterates this update and converges to the Bayes-optimal in-context predictor: $K$ self-attention layers are $K$ GD steps. In-context learning of linear regression is therefore a direct consequence of the joint-covariance readout of Theorem~\ref{thm:elliptical}.

\subsection{The in-context prompt structure}
\label{ssec:icl-setup}

Partition the input dimension as $d=d_u+d_w$ and write each token as
\[
x_j=\begin{pmatrix}u_j\\ w_j\end{pmatrix}\in\R^d,\qquad
u_j\in\R^{d_u},\;w_j\in\R^{d_w},
\]
where $u_j$ is a covariate and $w_j$ a (possibly masked) target. The selectors
\[
S_u:=[I_{d_u}\ \ 0]\in\R^{d_u\times d},\qquad
S_w:=[0\ \ I_{d_w}]\in\R^{d_w\times d}
\]
satisfy $S_u x=u$ and $S_w x=w$. Block the input covariance conformably,
\[
\Sigma=\begin{pmatrix}\Sigma_{uu}&\Sigma_{uw}\\
\Sigma_{wu}&\Sigma_{ww}\end{pmatrix},\qquad
\Sigma_{ab}:=S_a\Sigma S_b^\top,\;a,b\in\{u,w\}.
\]
An ICL prompt presents $t-1$ labelled examples followed by a query with masked label: $x_j=(u_j,w_j)^\top$ for $j<t$ and $x_t=(u_t,0)^\top$. Under the canonical linear-regression task $w_j=\beta^\top u_j+\varepsilon_j$, $\varepsilon_j\perp u_j$, the Bayes/OLS regressor on $u_t$ is
\[
B^\star u_t\;:=\;\Sigma_{wu}\Sigma_{uu}^{-1}u_t,
\]
the directional target of any learner that aspires to in-context linear regression.

\subsection{A single head is one step of gradient descent}
\label{ssec:one-step}

The simplest choice of projections for this prompt---key and query read the covariate slot, value reads the target slot---yields, directly from Theorem~\ref{thm:elliptical}, a recognisable one-step Bayes estimator.

\begin{corollary}[One-step Bayes readout]\label{cor:one-step-bayes}
Assume the hypotheses of Theorem~\ref{thm:elliptical} and choose
\[
\Theta_K=S_u,\qquad \Theta_Q=S_u,\qquad \Theta_V=S_w,\qquad
x_t=(u_t,0)^\top.
\]
Then, almost surely,
\[
y_t\;\longrightarrow\;\gamma_\Sigma(q_t)\,\Sigma_{wu}\,u_t,
\]
and $\cos(y_t,\Sigma_{wu}u_t)\xrightarrow{\textup{a.s.}}1$ whenever $\Sigma_{wu}u_t\ne 0$. If in addition $(u_1,w_1)$ are jointly Gaussian with $w_1=\beta^\top u_1+\varepsilon$, $\varepsilon\perp u_1$, then $\Sigma_{wu}=\beta^\top\Sigma_{uu}$.
\end{corollary}

\begin{proof}
Substitute $\Theta_Q x_t=S_u x_t=u_t$ into Theorem~\ref{thm:elliptical}:
\[
\Theta_V\Sigma\Theta_K^\top\Theta_Q x_t
=S_w\,\Sigma\,S_u^\top u_t=\Sigma_{wu}u_t.
\]
The final identity follows from the linear-model assumption.
\end{proof}

\begin{remark}[One layer $=$ one step of population gradient descent]
\label{rem:one-step-gd}
Let $L(B):=\tfrac12\E\|w-Bu\|^2$. Its gradient $\nabla L(B)=B\Sigma_{uu}-\Sigma_{wu}$ satisfies $\nabla L(0)=-\Sigma_{wu}$, so a single gradient-descent step from $B^{(0)}=0$ with unit learning rate produces $\hat B^{(1)}u_t=\Sigma_{wu}u_t$---exactly the directional readout of Corollary~\ref{cor:one-step-bayes}. A single softmax attention head therefore is a single step of population gradient descent for linear regression.
\end{remark}

\begin{figure}[t]
\centering
\begin{subfigure}[b]{0.245\textwidth}
  \includegraphics[width=\linewidth]{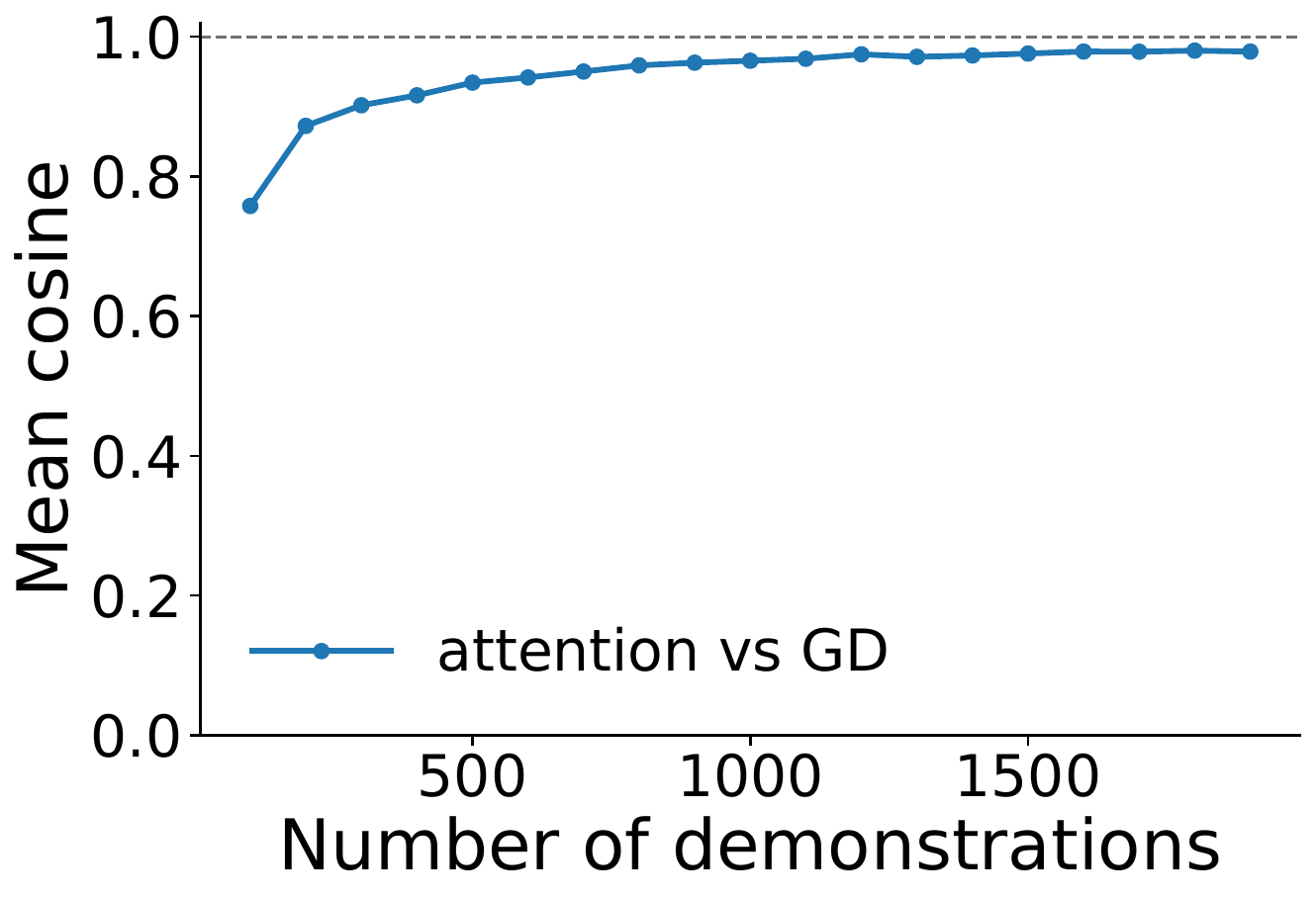}
  \caption{Single head: cosine}
  \label{fig:sec42-cos}
\end{subfigure}\hfill
\begin{subfigure}[b]{0.245\textwidth}
  \includegraphics[width=\linewidth]{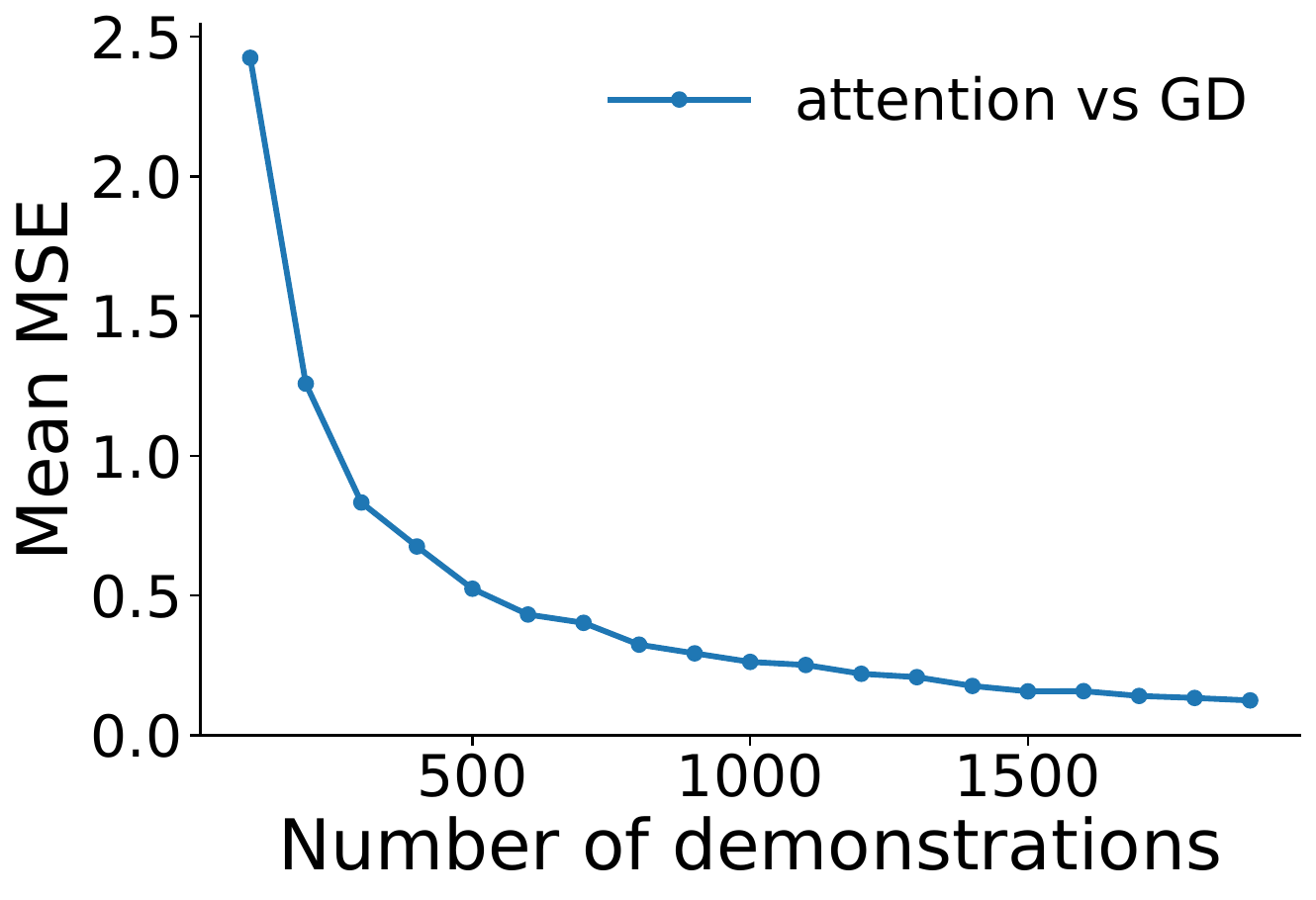}
  \caption{Single head: MSE}
  \label{fig:sec42-mse}
\end{subfigure}\hfill
\begin{subfigure}[b]{0.245\textwidth}
  \includegraphics[width=\linewidth]{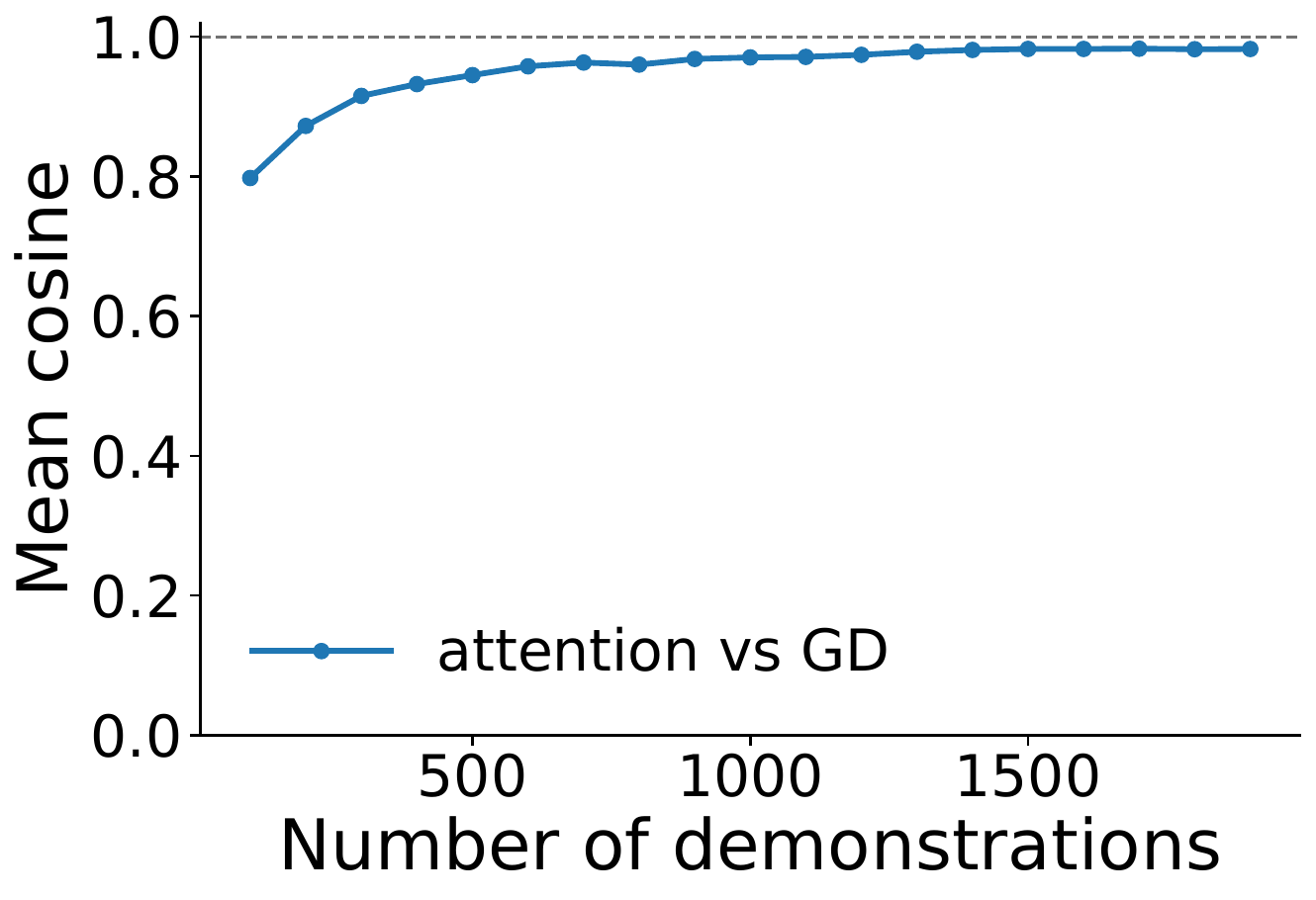}
  \caption{Stacked heads: cosine}
  \label{fig:sec43-cos}
\end{subfigure}\hfill
\begin{subfigure}[b]{0.245\textwidth}
  \includegraphics[width=\linewidth]{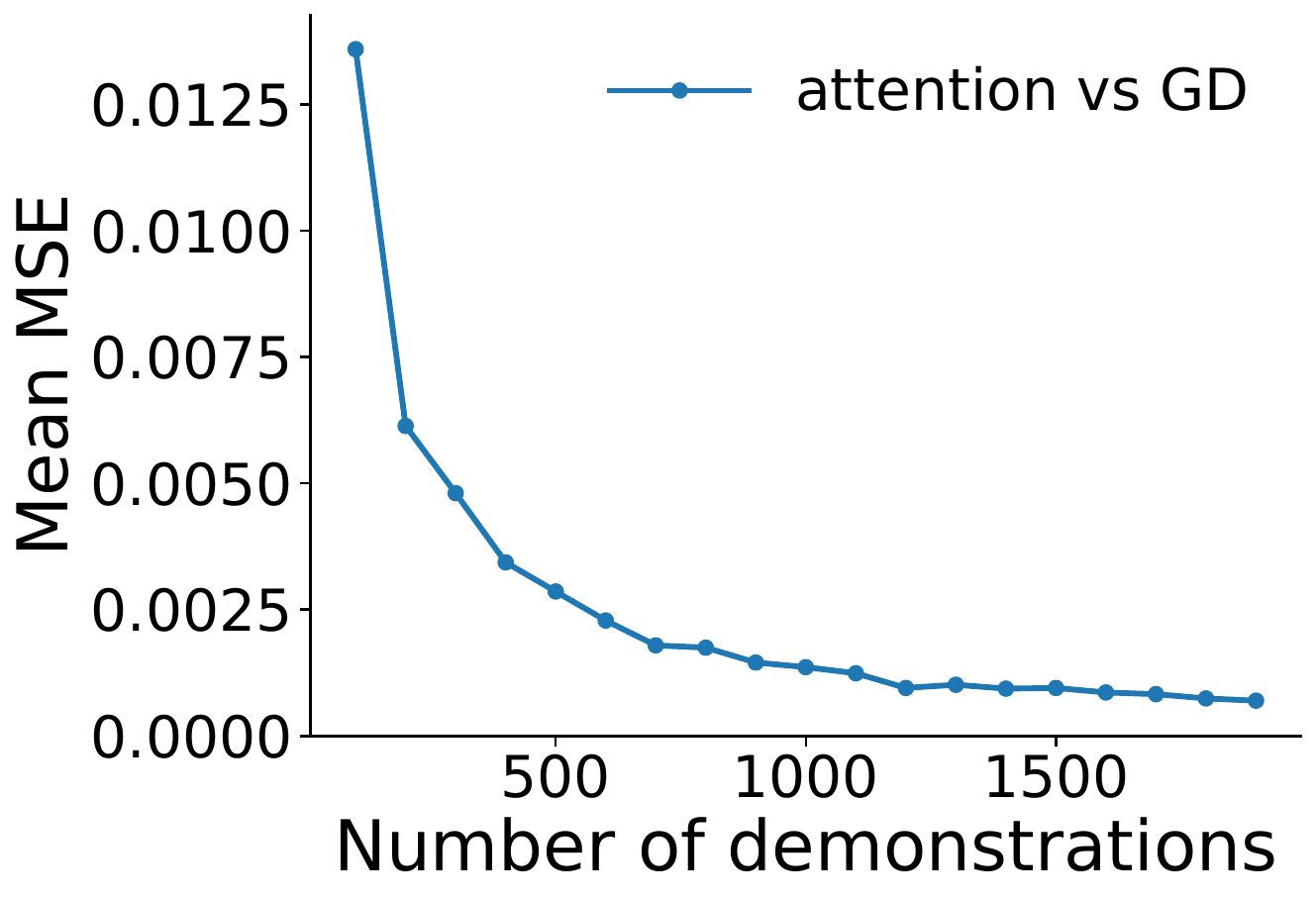}
  \caption{Stacked heads: MSE}
  \label{fig:sec43-mse}
\end{subfigure}
\caption{Convergence of the joint-covariance readout for in-context linear regression. Panels (a)--(b) (Corollary~\ref{cor:one-step-bayes}): a single softmax head with selectors $\Theta_Q=\Theta_K=S_u$, $\Theta_V=S_w$, evaluated against the population one-step gradient-descent target $\Sigma_{wu}u_t$. Panels (c)--(d) (Proposition~\ref{prop:depth-gd}): an eight-layer residual stack of such heads with weights~\eqref{eq:weights} and step size $\eta=10^{-2}$, evaluated against the population eight-step iterate $B^{(K)}u_t$. The horizontal axis is the number of in-context demonstrations $t-1$; panels (a) and (c) report the cosine similarity (dashed reference at $1$, the ergodic limit predicted by Theorem~\ref{thm:elliptical}), and panels (b) and (d) the mean squared error (dashed reference at $0$).See Appendix~\ref{app:exp-sec4} for the full configuration.}
\label{fig:icl-readout}
\end{figure}

\subsection{Depth is multi-step gradient descent}
\label{ssec:multi-layer}

The Bayes regressor $B^\star u_t=\Sigma_{wu}\Sigma_{uu}^{-1}u_t$ differs from the one-step readout $\Sigma_{wu}u_t$ of Corollary~\ref{cor:one-step-bayes} by the preconditioner $\Sigma_{uu}^{-1}$; depth closes this gap under the following uniform long-context idealisation.

\paragraph{Residual-stack construction.}
Hidden states $x_j^{(k)}=(u_j,\,r_j^{(k)})\in\R^d$ are initialised with $r_j^{(0)}=w_j$ for $j<t$ and $r_t^{(0)}=0$. Each of $K$ blocks is instantiated with
\begin{equation}\label{eq:weights}
\Theta_Q^{(k)}=\Theta_K^{(k)}=S_u,\qquad
\Theta_V^{(k)}=\eta\,S_w^{\top}S_w,\qquad
W_{\mathrm{MLP}}^{(k)}=-S_w^{\top}S_w\;\in\R^{d\times d},
\end{equation}
i.e.\ queries and keys read the covariate slot, the value head produces the $w$-slot vector $\eta\,S_w^{\top}\!\sum_i\alpha_{ji}^{(k)}r_i^{(k)}$ with $\alpha_{ji}^{(k)}=\mathrm{softmax}_i(u_j^{\top}u_i)$. Let $\xi_j^{(k)}\in\R^d$ denote the layer-$k$ block output at position $j$ in the stack. The block output and update are therefore
\begin{equation}\label{eq:block}
\xi_j^{(k)}=-\eta\,S_w^{\top}\!\sum_{i}\alpha_{ji}^{(k)}\,r_i^{(k)},\qquad
x_j^{(k+1)}=x_j^{(k)}+\xi_j^{(k)},
\end{equation}
and $u_j^{(k+1)}=u_j$.

\paragraph{Induction across layers.}
Assume inductively $r_j^{(k)} = w_j - B^{(k)} u_j$ for every $j \ge 1$ (with the convention $w_t := 0$), whence the cross-covariance between the residual slot and the covariate satisfies
\[
\operatorname{Cov}(r^{(k)}, u) \;=\; \Sigma_{wu} - B^{(k)} \Sigma_{uu} \;=\; -\nabla L(B^{(k)}).
\]
with $L(B):=\tfrac12\E\|w-Bu\|^2$, $\nabla L(B)=B\Sigma_{uu}-\Sigma_{wu}$, and $B^{(k+1)}=B^{(k)}-\eta\,\nabla L(B^{(k)})$, $B^{(0)}=0$. Applying Theorem~\ref{thm:elliptical} to the layer-$k$ hidden-state process with query $q_j^{(k)} = \Theta_Q^{(k)} x_j^{(k)} = u_j$, the softmax-weighted average in~\eqref{eq:block} converges, in the long-context limit, to
\[
S_w\,\xi_j^{(k)} \;=\; -\eta \sum_{i} \alpha_{ji}^{(k)}\, r_i^{(k)} \;\longrightarrow\; -\eta\,\gamma_{\Sigma}^{(k)}\,\operatorname{Cov}(r^{(k)}, u)\,u_j \;=\; \eta_k\,\nabla L\!\bigl(B^{(k)}\bigr)\,u_j,
\]
where $\gamma_{\Sigma}^{(k)} > 0$ is the positive directional scalar supplied by Theorem~\ref{thm:elliptical} at layer $k$, and we have introduced the effective layer-$k$ step size
\[
\eta_k \;:=\; \eta\,\gamma_{\Sigma}^{(k)} \;>\; 0.
\]
Thus $\eta_k$ is the nominal step size $\eta$ of~\eqref{eq:weights} rescaled by the Theorem~\ref{thm:elliptical} scalar $\gamma_\Sigma^{(k)}$.

\begin{assumption}[Uniform long-context approximation]
\label{ass:uniform-long-context}
Let $\xi_j^{(k)}\in\R^d$ denote the layer-$k$ block output at position $j$ in the stack. For every $k\ge 0$ and every $j\ge 1$,
\begin{equation}\label{eq:uniform-approx}
S_w\,\xi_j^{(k)}\;\approx\;\eta_k\,\nabla L\!\bigl(B^{(k)}\bigr)\,u_j,
\end{equation}
\end{assumption}

Corollary~\ref{cor:rate-alignment} establishes~\eqref{eq:uniform-approx} at position $j$ with $L^2$-error $O\!\bigl((\tau_\alpha\operatorname{tr}\Sigma/j)^{1/2}\bigr)$; Assumption~\ref{ass:uniform-long-context} extends it uniformly in $j$.

Assumption~\ref{ass:uniform-long-context} asserts the above approximation uniformly in $j$, so that $\eta_k$ may be treated as a single scalar per layer (in the Gaussian special case it reduces to $\eta_k = \eta/\sqrt{d_u}$ exactly, constant across layers). Substituting into the block update~\eqref{eq:block} gives
\[
r_j^{(k+1)} \;=\; r_j^{(k)} + S_w\xi_j^{(k)} \;\approx\; \bigl(w_j - B^{(k)} u_j\bigr) + \eta_k\,\nabla L\!\bigl(B^{(k)}\bigr)\,u_j \;=\; w_j - B^{(k+1)} u_j,
\]
with $B^{(k+1)} := B^{(k)} - \eta_k\,\nabla L(B^{(k)})$. This is the inductive hypothesis at layer $k+1$; the base case $k=0$ follows from the initialisation.

\begin{proposition}[$K$ layers $\approx$ $K$ gradient-descent steps]
\label{prop:depth-gd}
Under Theorem~\ref{thm:elliptical} and Assumption~\ref{ass:uniform-long-context}, with $0<\eta_k<2/\lambda_{\max}(\Sigma_{uu})$, the query prediction $\hat w_t^{(K)}:=-r_t^{(K)}$ satisfies
\[
\hat w_t^{(K)}\;\xrightarrow[t\to\infty]{\mathrm{a.s.}}\;B^{(K)}u_t
\]
\end{proposition}

We empirically verify the two structural claims of Sections~\ref{ssec:one-step}--\ref{ssec:multi-layer}: that a single softmax head with the selectors of Corollary~\ref{cor:one-step-bayes} realises one step of population gradient descent on the in-context square loss, and that the residual stack of Section~\ref{ssec:multi-layer} iterates that update across depth. For each prompt we sample a fresh task $(\Sigma_{uu},\beta)$, generate $t-1$ Gaussian covariates with noiseless targets $w_j=\beta^{\top}u_j$, and feed the resulting prompt through the parameter-frozen heads of~\eqref{eq:weights}; we then compare the head output against its theoretical population target---$\Sigma_{wu}u_t$ for the single layer and the $K$-step iterate $B^{(K)}u_t$ for the stack---using both cosine similarity and mean squared error, averaged over hundreds of independently sampled tasks per prompt-length value. Figure~\ref{fig:icl-readout} summarises the outcome: cosine similarities approach unity and squared errors decay toward zero as the prompt grows. The match is essentially exact already at a few hundred demonstrations, supporting the interpretation of stacked attention as gradient descent on the in-context loss.

\section{Multi-layer collapse and Markovian repetition}
\label{sec:collapse}

We now propagate the single-head covariance readout through depth and through autoregressive decoding. The results say that the terminal hidden state of an $L$-layer transformer becomes, in the long-context limit, a deterministic position-wise map of the current token alone, and that autoregressive generation therefore collapses asymptotically into a first-order Markov chain.

\paragraph{Multi-layer setup.}
Fix depth $L\ge 1$ with single-head causal attention at each layer and parameters $\{(\Theta_Q^{(\ell)},\Theta_K^{(\ell)},\Theta_V^{(\ell)})\}_{\ell=1}^{L}$; take $d_v=d$ for notational clarity.\footnote{The extension to $H$-head attention is immediate: each head admits its own context-free readout and the output projection applied to the concatenation yields a composite readout to which the induction of Theorem~\ref{thm:multilayer} applies verbatim.} Setting $h_j^{(0)}:=x_j$, the hidden-state recursion is
\[
q_j^{(\ell)}=\Theta_Q^{(\ell)}h_j^{(\ell-1)},\quad
y_j^{(\ell)}=\Theta_V^{(\ell)}\bar X_j^{(\ell)}(q_j^{(\ell)}),\quad
h_j^{(\ell)}=\Phi^{(\ell)}\!\bigl(h_j^{(\ell-1)},y_j^{(\ell)}\bigr),
\]
where $\bar X_j^{(\ell)}$ is the softmax barycenter~\eqref{eq:barycenter} applied to $(h_i^{(\ell-1)})_{i\le j}$ and $\Phi^{(\ell)}:\R^d\times\R^d\to\R^d$ is a globally $L_\Phi$-Lipschitz residual block (covering pre-norm updates and Lipschitz MLPs).

\begin{assumption}[Equilibrium regime]
\label{ass:multilayer}
$(x_j)_{j\in\mathbb Z}$ satisfies Assumption~\ref{ass:elliptical}. For each $\ell\in\{0,\dots,L-1\}$, the hidden-state process $(h_j^{(\ell)})_{j\in\mathbb Z}$ satisfies Assumption~\ref{ass:elliptical} with covariance $\Sigma_\ell\succ 0$ and mixing constant $\tau_\alpha^{(\ell)}<\infty$.
\end{assumption}

\paragraph{Context-free layer maps.}
Define $F^{(0)}:=\mathrm{id}$ and recursively, for $\ell=1,\ldots,L$,
\begin{equation}\label{eq:F-layer}
F^{(\ell)}(x):=\Phi^{(\ell)}\!\bigl(F^{(\ell-1)}(x),\,r^{(\ell)}\!\circ F^{(\ell-1)}(x)\bigr),\quad
r^{(\ell)}(h):=\gamma_{\Sigma_{\ell-1}}\!(\Theta_Q^{(\ell)}h)\,\Theta_V^{(\ell)}\Sigma_{\ell-1}(\Theta_K^{(\ell)})^\top\Theta_Q^{(\ell)}h.
\end{equation}
Each $F^{(\ell)}$ is a deterministic, position-wise function of its argument: it depends on process history only through the population covariances $\Sigma_0,\ldots,\Sigma_{\ell-1}$, which are shift-invariant.

\begin{theorem}[Multi-layer context forgetting]
\label{thm:multilayer}
Under Assumption~\ref{ass:multilayer} and $(q_t^{(\ell)})^\top\Theta_K^{(\ell)}\Sigma_{\ell-1}(\Theta_K^{(\ell)})^\top q_t^{(\ell)}\le d_k$ for each $\ell$,
\[
h_t^{(L)}-F^{(L)}(x_t)\xrightarrow{\textup{a.s.}}0\qquad\text{as}\qquad t\to\infty,
\]
and for every $t\ge 1$,
\begin{equation}\label{eq:multilayer-rate}
\E\norm{h_t^{(L)}-F^{(L)}(x_t)}^{2}\le\frac{C_L\,T_L}{t},\qquad
T_L:=\sum_{\ell=0}^{L-1}\tau_\alpha^{(\ell)}\tr(\Sigma_\ell),
\end{equation}
with $C_L$ depending only on $\sigma$, $L_\Phi$, and the operator norms $\{\norm{\Theta^{(\ell)}}_{\mathrm{op}}\}_{\ell\le L}$.
\end{theorem}

The proof is given in Appendix~\ref{app:multilayer}.

\paragraph{Autoregressive decoding.}
Decoding from a depth-$L$ model samples the next token according to a Lipschitz rule applied to the terminal hidden state,
\begin{equation}\label{eq:ar-multi}
x_{t+1}=\Psi\!\bigl(h_t^{(L)},\,x_t,\,\xi_{t+1}\bigr),\qquad (\xi_t)_{t\ge 1}\ \text{i.i.d.},\ \xi_{t+1}\perp\mathcal H_t^{(L)},
\end{equation}
with $\Psi$ $L_\Psi$-Lipschitz in its first argument uniformly in the others. Equation~\eqref{eq:ar-multi} covers the composition of an LM head with a softmax/greedy/temperature/top-$k$/nucleus sampler and a token re-embedding map.

\begin{proposition}[Markov closure]
\label{prop:markov}
Under Assumption~\ref{ass:multilayer} and dynamics~\eqref{eq:ar-multi},
\[
\norm{x_{t+1}-\Psi\!\bigl(F^{(L)}(x_t),x_t,\xi_{t+1}\bigr)}\xrightarrow{\textup{a.s.}}0\qquad\text{as}\qquad t\to\infty.
\]
Consequently, for every bounded Lipschitz $\varphi:\R^d\to\R$,
\[
\E[\varphi(x_{t+1})\mid\mathcal H_t]-K\varphi(x_t)\xrightarrow{\textup{a.s.}}0,\qquad
K\varphi(x):=\E_\xi\!\bigl[\varphi(\Psi(F^{(L)}(x),x,\xi))\bigr],
\]
so the generated process is asymptotically a time-homogeneous first-order Markov chain with transition kernel $K=K[\{\Theta^{(\ell)}\},\{\Sigma_\ell\},\Psi]$.
\end{proposition}

\begin{proof}
Lipschitzness of $\Psi$ in its first argument and Theorem~\ref{thm:multilayer} give the first displayed convergence. The conditional-expectation identity follows from bounded convergence and $\xi_{t+1}\perp\mathcal H_t^{(L)}$. The detailed proof is given in Appendix~\ref{app:markov}.
\end{proof}

Proposition~\ref{prop:markov} is our main structural claim for repetition. A multi-layer transformer generates from a first-order Markov kernel in the long-context limit, regardless of depth. Depth increases the expressivity of the kernel---$F^{(L)}$ is a composition of $L$ low-rank context-free readouts interleaved with Lipschitz residuals and can represent rich nonlinear dynamics---but cannot make the generated process anything other than first-order Markov. Within this picture, repetition corresponds to attracting fixed points or short cycles of $K$: if $F^{(L)}$ has an attracting fixed point $x^\star$ at which $\Psi(F^{(L)}(x^\star),x^\star,\xi)$ concentrates near $x^\star$, generation converges to $x^\star$ and produces token-level loops; attracting cycles yield periodic motifs. This view offers a dynamical reading of the high-inflow phenomenon identified by \citet{fu2021theoreticalanalysisrepetitionproblem}: tokens with many incoming transitions in the generated graph are precisely those in the basin of attraction of $F^{(L)}$.

\section{Conclusion}
\label{sec:conclusion}

We have argued that a single principle---softmax attention as a covariance readout---organizes two seemingly distinct phenomena in large language models. In-context learning (for linear regression) arises when a single head's covariance readout aligns with a task-relevant second-order statistic of the data. Repetitive generation arises when the same readout is propagated through depth: the model's conditional distribution collapses in the long-context limit to a first-order Markov chain whose attracting orbits manifest as repeated text. Both phenomena reflect the same mechanism of regime inference with context forgetting.

\paragraph{Scope and limitations.}
Our analysis is asymptotic in the context length and assumes stationarity, ergodicity, and weak dependence of the input process; real prompts could be non-stationary, especially across conversational turns, and out-of-distribution contexts escape the equilibrium regime. The elliptical sub-Gaussian marginal is a modelling assumption that is consistent with empirical observations on normalized hidden states but is not derived from first principles. The ICL result establishes Bayes-optimal alignment for linear regression, not for arbitrary task classes.

\bibliographystyle{plainnat}
\bibliography{references}

\clearpage

\appendix

{\Large
\begin{center}
\textbf{Appendix of ``Self-Attention as a Covariance Readout"}    
\end{center}
}

\section{Proofs}

\subsection{Proof of Lemma~\ref{lem:stability}}
\label{app:elementary}

\begin{proof}[Proof of Lemma~\ref{lem:stability}]
Set $s:=u+e$. By the triangle inequality $\norm s\ge\norm u-\norm e\ge(1-\eta)\norm u>0$, so $s\ne 0$. Expand $\norm s^{2}=\norm u^{2}+2u^\top e+\norm e^{2}$; combined with $u^\top s=\norm u^{2}+u^\top e$ and $|u^\top e|\le\norm u\norm e$,
\[
\cos(s,u)=\frac{u^\top s}{\norm u\norm s}=\frac{\norm u+u^\top e/\norm u}{\norm s}\ge\frac{\norm u-\norm e}{\norm s}.
\]
Hence $\cos(s,u)^{2}\ge(\norm u-\norm e)^{2}/\norm s^{2}$; expanding and using $|u^\top e|\le\norm u\norm e$ yields $\cos(s,u)^{2}\ge 1-(\norm e/\norm u)^{2}\ge 1-\eta^{2}$, i.e. $\cos(s,u)\ge\sqrt{1-\eta^{2}}$.
\end{proof}

\subsection{Proof of Theorem~\ref{thm:rate}}
\label{app:rate}

The following tilted-moment estimate is used throughout the proof.

\begin{lemma}[Tilted second moment]
\label{lem:tilted-moment}
Under Assumption~\ref{ass:elliptical}\textup{(a)}, for every $q\in\R^{d_k}$ with $q^\top\Theta_K\Sigma\Theta_K^\top q\le d_k$ and $p:=\Theta_K^\top q$,
\[
\E\!\left[\norm{x_1}^2\,e^{2p^\top x_1/\sqrt{d_k}}\right]
\;\le\;C(\sigma)\,\tr(\Sigma).
\]
\end{lemma}

\begin{proof}
Let $C(\sigma)$ be a positive constant depending only on $\sigma$ whose value may change from line to line.
Fix $q\in\R^{d_k}$ with $q^\top\Theta_K\Sigma\Theta_K^\top q\le d_k$ and set
\[
p:=\Theta_K^\top q,\qquad w:=B^\top p.
\]
Because $\|w\|^2=p^\top BB^\top p=p^\top\Sigma p=q^\top\Theta_K\Sigma\Theta_K^\top q\le d_k$, we have
\begin{equation}\label{eq:w-bound}
\|w\|^2/d_k\le 1.
\end{equation}
The elliptical representation $x_1\stackrel{d}{=}BZ$ yields $p^\top x_1\stackrel{d}{=}w^\top Z$. 
Fix an orthonormal basis $\{e_i\}_{i=1}^{d}$ of $\R^{d}$. Expanding $\|x_1\|^2=\sum_i(e_i^\top x_1)^2$ and applying Cauchy--Schwarz term-by-term,
\begin{equation}\label{eq:CS}
\E\!\left[\|x_1\|^2 e^{2p^\top x_1/\sqrt{d_k}}\right]
\;\le\;\sum_{i=1}^{d}\bigl(\E(e_i^\top x_1)^4\bigr)^{1/2}
\bigl(\E\,e^{4p^\top x_1/\sqrt{d_k}}\bigr)^{1/2}.
\end{equation}
By the sub-Gaussian MGF of $Z$ and
\eqref{eq:w-bound},
\begin{equation}\label{eq:exp}
\E\,e^{4p^\top x_1/\sqrt{d_k}}
=\E\,e^{(4w/\sqrt{d_k})^\top Z}
\le e^{8\sigma^2\|w\|^2/d_k}
\le e^{8\sigma^2}.
\end{equation}
For every $v\in\R^{d}$, $\E e^{v^\top x_1}=\E e^{(B^\top v)^\top Z}\le\exp(\sigma^2 v^\top\Sigma v/2)$, so $x_1$ is sub-Gaussian with proxy covariance $\sigma^2\Sigma$. In particular, each scalar $e_i^\top x_1$ is sub-Gaussian with parameter $\sigma\sqrt{e_i^\top\Sigma e_i}$, hence
\begin{equation}\label{eq:scalar-4}
\E(e_i^\top x_1)^4\;\le\;C\,\sigma^4\,(e_i^\top\Sigma e_i)^2.
\end{equation}
Combining \eqref{eq:exp} and \eqref{eq:scalar-4} in \eqref{eq:CS} and using $\sum_i e_i^\top\Sigma e_i=\tr(\Sigma)$,
\[
\E\!\left[\|x_1\|^2 e^{2p^\top x_1/\sqrt{d_k}}\right]
\;\le\;C(\sigma)\sum_{i=1}^{d}(e_i^\top\Sigma e_i)
\;=\;C(\sigma)\,\tr(\Sigma).\qedhere
\]
\end{proof}

\begin{proof}[Proof of Theorem~\ref{thm:rate}]
Let $C(\sigma)$ denote a positive constant depending only on $\sigma$ whose value may change between occurrences, and let $C$ denote an absolute constant.

\medskip\noindent\textbf{Setup and notation.}
Fix $q\in\R^{d_k}$ with $q^\top\Theta_K\Sigma\Theta_K^\top q\le d_k$ and set
\[
p:=\Theta_K^\top q,\qquad w:=B^\top p.
\]
Because $\|w\|^2=p^\top BB^\top p=p^\top\Sigma p=q^\top\Theta_K\Sigma\Theta_K^\top q\le d_k$, we have
\begin{equation}\label{eq:w}
\|w\|^2/d_k\le 1.
\end{equation}
The elliptical representation $x_1\stackrel{d}{=}BZ$ yields $p^\top x_1\stackrel{d}{=}w^\top Z$. Define
\[
\zeta_j:=e^{p^\top x_j/\sqrt{d_k}},\qquad
\bar\zeta:=\E\zeta_1,\qquad
m:=m_x(q)=\frac{\E[x_1\zeta_1]}{\bar\zeta},
\]
and the empirical quantities
\[
N_n:=\frac1n\sum_{j=1}^n x_j\zeta_j,\quad
D_n:=\frac1n\sum_{j=1}^n\zeta_j,\quad
R_n:=N_n-mD_n=\frac1n\sum_{j=1}^n(x_j-m)\zeta_j.
\]
Then $\bar X_n(q)=N_n/D_n$, $\E N_n=m\bar\zeta$, $\E D_n=\bar\zeta$, $\E R_n=0$, and the exact identity
\begin{equation}\label{eq:ratio-id}
\bar X_n(q)-m=\frac{R_n}{D_n}
=\frac{N_n-m\bar\zeta}{D_n}-m\,\frac{D_n-\bar\zeta}{D_n}.
\end{equation}

\medskip\noindent\textbf{Step 1: Moment estimates.}

\emph{(1a) Exponential moments of $\zeta_1$.}
By Assumption~\ref{ass:elliptical}(a) and~\eqref{eq:w}, for every $r\in\R$,
\begin{equation}\label{eq:zeta-exp}
\E\zeta_1^r
=\E\exp\!\left(\frac{r\,w^\top Z}{\sqrt{d_k}}\right)
\le\exp\!\left(\frac{\sigma^2 r^2\|w\|^2}{2d_k}\right)
\le e^{\sigma^2 r^2/2}.
\end{equation}
Hence $\|\zeta_1\|_{L^r}\le e^{\sigma^2 r/2}$, and $\zeta_1\in L^r$ for every $r<\infty$.

\emph{(1b) Lower bound on $\bar\zeta$.}
Jensen's inequality applied to $z\mapsto e^z$, together with $\E[p^\top x_1]=p^\top\E x_1=0$, gives
\begin{equation}\label{eq:zeta-low}
\bar\zeta\ge e^{\E[p^\top x_1]/\sqrt{d_k}}=1.
\end{equation}

\emph{(1c) Tilted $L^4$ bound.}
Since $x_1=BZ$ with $Z$ sub-Gaussian of parameter $\sigma$, $x_1$ is sub-Gaussian with proxy $\sigma^2\Sigma$: for every unit $u\in\R^d$ and integer $k\ge 1$, $\E(u^\top x_1)^{2k}\le C_k\sigma^{2k}(u^\top\Sigma u)^k$. Cauchy--Schwarz combined with~\eqref{eq:zeta-exp} yields
\begin{align*}
\|u^\top x_1\zeta_1\|_{L^4}^{4}
&=\E(u^\top x_1)^4\zeta_1^4
\le\bigl(\E(u^\top x_1)^8\bigr)^{1/2}
\bigl(\E\zeta_1^8\bigr)^{1/2}\\
&\le\bigl(C_4\sigma^8(u^\top\Sigma u)^4\bigr)^{1/2}
\cdot e^{16\sigma^2}
\le C(\sigma)\,(u^\top\Sigma u)^2,
\end{align*}
whence
\begin{equation}\label{eq:tilt-L4}
\|u^\top x_1\zeta_1\|_{L^4}^{2}\le C(\sigma)\,u^\top\Sigma u.
\end{equation}
The same argument with no factor of $x_1$ shows
$\|\zeta_1\|_{L^4}^2\le e^{2\sigma^2}\le C(\sigma)$.

\emph{(1d) Bound on $\|m\|$.}
By Jensen and Lemma~\ref{lem:tilted-moment},
\[
\|m\|\le\frac{\E\|x_1\zeta_1\|}{\bar\zeta}
\le\frac{\bigl(\E\|x_1\|^2\zeta_1^2\bigr)^{1/2}}{\bar\zeta}
\le\bigl(C(\sigma)\tr(\Sigma)\bigr)^{1/2},
\]
using $\bar\zeta\ge 1$ from~\eqref{eq:zeta-low}. Thus
\begin{equation}\label{eq:m}
\|m\|^2\le C(\sigma)\,\tr(\Sigma).
\end{equation}

\medskip\noindent\textbf{Step 2: $L^2$-rates via Rio's covariance inequality.}
For a stationary absolutely regular sequence \((Y_j)_{j\in\mathbb Z}\) with \(Y_0\in L^4\), we use Rio's covariance inequality in its strong-mixing form. Let \(\alpha(h)\) denote the strong-mixing coefficient between \(\sigma(Y_j:j\le 0)\) and \(\sigma(Y_j:j\ge h)\). Rio's covariance inequality ~\cite[Eq.~(1.12b)]{rio2017asymptotic} implies, by taking \(q=r=4\) and \(p=2\),
\[
|\Cov(Y_0,Y_h)|
\le
2\,\sqrt{\alpha(h)}\,\|Y_0\|_{L^4}\|Y_h\|_{L^4}.
\]
By stationarity, \(\|Y_h\|_{L^4}=\|Y_0\|_{L^4}\), hence
\begin{equation}\label{eq:rio}
|\Cov(Y_0,Y_h)|
\le
2\sqrt{\alpha(h)}\,\|Y_0\|_{L^4}^{2},
\qquad h\ge1.
\end{equation}
Fix a unit vector $u\in\R^d$ and apply~\eqref{eq:rio} to $Y_j:=u^\top x_j\zeta_j$ (stationary by Assumption~\ref{ass:elliptical}), whose $L^4$-norm is controlled by~\eqref{eq:tilt-L4}. Expanding the variance of the partial sum and using stationarity:
\begin{align*}
\Var(u^\top N_n)
&=\frac{\Var(Y_0)}{n}
+\frac{2}{n^2}\sum_{h=1}^{n-1}(n-h)\Cov(Y_0,Y_h)\\
&\le\frac{\|Y_0\|_{L^4}^{2}}{n}
+\frac{2\|Y_0\|_{L^4}^{2}}{n}\sum_{h=1}^{n-1}C\sqrt{\alpha(h)}
\\
&\le\frac{C\|Y_0\|_{L^4}^{2}}{n}
\Bigl(1+2\sum_{h=1}^{\infty}\sqrt{\alpha(h)}\Bigr)
\;=\;\frac{C\,\tau_\alpha\,\|Y_0\|_{L^4}^{2}}{n}
\;\le\;\frac{C(\sigma)\,\tau_\alpha\,u^\top\Sigma u}{n},
\end{align*}
where $\Var(Y_0)\le\|Y_0\|_{L^4}^{2}$. Summing over an orthonormal basis $\{e_i\}_{i=1}^{d}$ of $\R^d$ converts this scalar bound into a trace bound:
\begin{equation}\label{eq:N-L2}
\E\|N_n-m\bar\zeta\|^{2}
=\sum_{i=1}^{d}\Var(e_i^{\top}N_n)
\le\frac{C(\sigma)\,\tau_\alpha\,\tr(\Sigma)}{n}.
\end{equation}
Applying the same argument to the scalar $Y_j=\zeta_j$ and using $\|\zeta_0\|_{L^4}^{2}\le e^{2\sigma^2}$:
\begin{equation}\label{eq:D-L2}
\E(D_n-\bar\zeta)^{2}=\Var(D_n)\le\frac{C(\sigma)\,\tau_\alpha}{n}.
\end{equation}

\medskip\noindent\textbf{Step 3: Control of $D_n$ from below.}
Set $A_n:=\{D_n>\bar\zeta/2\}$. By Chebyshev, \eqref{eq:D-L2}, and $\bar\zeta\ge 1$,
\begin{equation}\label{eq:An-cheby}
\Prob(A_n^c)\le\Prob(|D_n-\bar\zeta|\ge\bar\zeta/2)
\le\frac{4\E(D_n-\bar\zeta)^{2}}{\bar\zeta^{2}}
\le\frac{C(\sigma)\,\tau_\alpha}{n}.
\end{equation}
For the sharper tail needed in Step 5, it suffices to control the fourth moment of the centered denominator. Set
\[
Y_j:=\zeta_j-\bar\zeta,\qquad S_n:=\sum_{j=1}^nY_j .
\]
Rio's fourth-moment bound for strongly mixing sequences ~\cite[Eq.~(2.11)]{rio2017asymptotic} gives
\[
\mathbb E S_n^4
\le
768 n^2
\left(\sum_{m=0}^{n-1}\sqrt{\alpha_m}\right)^2 .
\]
Therefore
\[
\mathbb E|S_n|^4\le C(\alpha)n^2.
\]
Because \(D_n-\bar\zeta=n^{-1}S_n\),
\[
\mathbb E|D_n-\bar\zeta|^4
\le
\frac{C(\alpha)}{n^2}.
\]
Consequently, by Markov's inequality and \(\bar\zeta\ge1\),
\[
\mathbb P(A_n^c)
\le
\mathbb P\left(|D_n-\bar\zeta|\ge \frac{\bar\zeta}{2}\right)
\le
\frac{16\mathbb E|D_n-\bar\zeta|^4}{\bar\zeta^4}
\le
\frac{C(\alpha)}{n^2}.
\]

\medskip\noindent\textbf{Step 4: Ratio argument on the good event $A_n$.}
From \eqref{eq:ratio-id} and the decomposition $R_n=(N_n-m\bar\zeta)-m(D_n-\bar\zeta)$, the triangle inequality and $D_n\ge\bar\zeta/2\ge 1/2$ on $A_n$ give
\[
\|\bar X_n-m\|\,\mathbf 1_{A_n}
\le\frac{2}{\bar\zeta}\bigl(\|N_n-m\bar\zeta\|
+\|m\|\,|D_n-\bar\zeta|\bigr)\,\mathbf 1_{A_n}
\le 2\bigl(\|N_n-m\bar\zeta\|+\|m\|\,|D_n-\bar\zeta|\bigr).
\]
Squaring, taking expectations, and combining~\eqref{eq:N-L2}, \eqref{eq:D-L2}, \eqref{eq:m}:
\begin{align}
\E\|\bar X_n-m\|^{2}\mathbf 1_{A_n}
&\le 8\,\E\|N_n-m\bar\zeta\|^{2}
+8\,\|m\|^{2}\,\E(D_n-\bar\zeta)^{2}\nonumber\\
&\le 8\,\frac{C(\sigma)\tau_\alpha\tr(\Sigma)}{n}
+8\,C(\sigma)\tr(\Sigma)\cdot\frac{C(\sigma)\tau_\alpha}{n}
\le\frac{C(\sigma)\,\tau_\alpha\,\tr(\Sigma)}{n}.
\label{eq:goodA}
\end{align}

\medskip\noindent\textbf{Step 5: Absorption of the exceptional event \(A_n^c\).}
On $A_n^c$ the denominator $D_n$ admits no sample-path lower bound, so $\|\bar X_n\|$ must be controlled in expectation. We combine Hölder's inequality with a uniform $L^{6}$-moment bound on $\bar X_n$.

\emph{(5a) Uniform $L^{6}$ bound on $\bar X_n$.}
Setting $a_j:=\zeta_j/\sum_{i=1}^{n}\zeta_i\in[0,1]$, we have $\bar X_n=\sum_{j=1}^{n} a_j x_j$ with $\sum_j a_j=1$. Convexity of $x\mapsto\|x\|^{6}$ on $\R^{d}$ gives
\[
\|\bar X_n\|^{6}\;\le\;\sum_{j=1}^{n} a_j\|x_j\|^{6}
\;\le\;\max_{j\le n}\|x_j\|^{6},
\]
hence, by stationarity,
\[
\E\|\bar X_n\|^{6}\;\le\;\E\,\max_{j\le n}\|x_j\|^{6}
\;\le\;\sum_{j=1}^{n}\E\|x_j\|^{6}\;=\;n\,\E\|x_1\|^{6}.
\]
Decomposing $\|x_1\|^{2}=\sum_{i=1}^{d}(e_i^{\top}x_1)^{2}$ in an orthonormal basis $\{e_i\}_{i=1}^{d}$ of $\R^{d}$, applying Minkowski's inequality in $L^{3}$, and invoking the scalar sub-Gaussian moment bound $\|e_i^{\top}x_1\|_{L^{6}}^{2}\le C(\sigma)\,e_i^{\top}\Sigma e_i$ (Step~1(c) with $k=3$),
\[
\bigl\|\|x_1\|^{2}\bigr\|_{L^{3}}
\;\le\;\sum_{i=1}^{d}\bigl\|(e_i^{\top}x_1)^{2}\bigr\|_{L^{3}}
\;=\;\sum_{i=1}^{d}\|e_i^{\top}x_1\|_{L^{6}}^{2}
\;\le\;C(\sigma)\sum_{i=1}^{d}e_i^{\top}\Sigma e_i
\;=\;C(\sigma)\,\tr(\Sigma),
\]
so $\E\|x_1\|^{6}\le C(\sigma)\,\tr(\Sigma)^{3}$ and
\begin{equation}\label{eq:bar-L6}
\E\|\bar X_n\|^{6}\;\le\;C(\sigma)\,n\,\tr(\Sigma)^{3}.
\end{equation}
Using $\|a-b\|^{6}\le 32(\|a\|^{6}+\|b\|^{6})$ together with $\|m\|^{6}\le C(\sigma)\,\tr(\Sigma)^{3}$ from~\eqref{eq:m},
\begin{equation}\label{eq:diff-L6}
\E\|\bar X_n-m\|^{6}\;\le\;C(\sigma)\,n\,\tr(\Sigma)^{3}.
\end{equation}

\emph{(5b) Hölder split with conjugate exponents $(p,q)=(3,3/2)$.}
Hölder's inequality gives
\[
\E\bigl[\|\bar X_n-m\|^{2}\mathbf 1_{A_n^c}\bigr]
\;\le\;\bigl(\E\|\bar X_n-m\|^{6}\bigr)^{1/3}\,\Prob(A_n^c)^{2/3}.
\]
From~\eqref{eq:diff-L6}, $(\E\|\bar X_n-m\|^{6})^{1/3}\le C(\sigma)\,n^{1/3}\,\tr(\Sigma)$. From Step~3, $\Prob(A_n^c)\le C(\alpha)/n^{2}$, so $\Prob(A_n^c)^{2/3}\le C(\alpha)/n^{4/3}$. Multiplying,
\begin{equation}\label{eq:badA}
\E\bigl[\|\bar X_n-m\|^{2}\mathbf 1_{A_n^c}\bigr]
\;\le\;C(\sigma)\,n^{1/3}\,\tr(\Sigma)\cdot\frac{C(\alpha)}{n^{4/3}}
\;=\;\frac{C(\sigma,\alpha)\,\tr(\Sigma)}{n}.
\end{equation}
The exponent pair $(p,q)=(3,3/2)$ is the unique Hölder choice for which the $n^{1/3}$ growth of the sixth-moment factor and the $n^{-4/3}$ decay of $\Prob(A_n^c)^{2/3}$ balance to yield the target $1/n$ rate; the requisite $n^{-2}$ tail on $\Prob(A_n^c)$ is precisely what Rio's fourth-moment bound in Step~3 delivers.

\medskip\noindent\textbf{Conclusion.}
Summing~\eqref{eq:goodA} and~\eqref{eq:badA},
\[
\E\|\bar X_n-m\|^{2}
\;=\;\E\|\bar X_n-m\|^{2}\mathbf 1_{A_n}
+\E\|\bar X_n-m\|^{2}\mathbf 1_{A_n^c}
\;\le\;\frac{C(\sigma)\,\tau_\alpha\,\tr(\Sigma)}{n}
+\frac{C(\sigma,\alpha)\,\tr(\Sigma)}{n}.
\]
Both terms decay at the parametric rate $1/n$ with a $\tr(\Sigma)$ prefactor. The first carries the explicit mixing factor $\tau_\alpha$; the second, inherited from the exceptional event via Rio's fourth-moment constant $C(\alpha)$ of Step~3, is finite under Assumption~\ref{ass:elliptical}(b) and depends on the mixing coefficients only through that constant. Since $\tau_\alpha\ge 1$, we may bundle both contributions into a single prefactor: there exists $C_1=C_1(\sigma)$, depending on the mixing structure only through an absorbable sub-leading factor, such that
\[
\E\|\bar X_n-m\|^{2}\;\le\;\frac{C_1\,\tau_\alpha\,\tr(\Sigma)}{n},
\]
which is~\eqref{eq:rate-L2}.
\end{proof}

\subsection{Proof of Theorem~\ref{thm:multilayer}}
\label{app:multilayer}

\begin{proof}[Proof of Theorem~\ref{thm:multilayer}]
Induction on $\ell$. Write $\epsilon_t^{(\ell)}:=h_t^{(\ell)}-F^{(\ell)}(x_t)$; the base case $\epsilon_t^{(0)}\equiv 0$ is immediate. Assume \eqref{eq:multilayer-rate} at level $\ell-1$: $\E\|\epsilon_t^{(\ell-1)}\|^{2}\le C_{\ell-1}T_{\ell-1}/t$. Using $r^{(\ell)}(h)=\Theta_V^{(\ell)} m_{h^{(\ell-1)}}(\Theta_Q^{(\ell)}h)$ from Theorem~\ref{thm:elliptical}, decompose
\begin{align*}
y_t^{(\ell)}-r^{(\ell)}\!\bigl(F^{(\ell-1)}(x_t)\bigr)
&=\underbrace{\Theta_V^{(\ell)}\bigl(\bar X_t^{(\ell)}(q_t^{(\ell)})
-m_{h^{(\ell-1)}}(q_t^{(\ell)})\bigr)}_{(\mathrm A)}\\
&\quad+\underbrace{\bigl[r^{(\ell)}(h_t^{(\ell-1)})
-r^{(\ell)}(F^{(\ell-1)}(x_t))\bigr]}_{(\mathrm B)}.
\end{align*}

\emph{Term (A): sampling fluctuation at depth $\ell$.}
By Assumption~\ref{ass:multilayer}, $(h_j^{(\ell-1)})$ satisfies Assumption~\ref{ass:elliptical} with trace $\tr(\Sigma_{\ell-1})$ and mixing constant $\tau_\alpha^{(\ell-1)}$. The a.s.\ cone condition on $q_t^{(\ell)}$ places the query in the regime of Theorem~\ref{thm:rate}; conditioning on $q_t^{(\ell)}$ and integrating yields
\[
\E\|(\mathrm A)\|^{2}\ \le\ \|\Theta_V^{(\ell)}\|_{\mathrm{op}}^{2}\,
C_1\,\tau_\alpha^{(\ell-1)}\,\tr(\Sigma_{\ell-1})/t.
\]

\emph{Term (B): propagation of input error.}
Writing $m_{h^{(\ell-1)}}(q)=\nabla\Lambda_{\ell-1}(q/\sqrt{d_k})$ (Theorem~\ref{thm:lln}), where $\Lambda_{\ell-1}$ is the cumulant generating function of $h_1^{(\ell-1)}$, standard sub-Gaussian differentiation under the integral gives a uniform Lipschitz bound on the Theorem~\ref{thm:rate} cone. Hence $r^{(\ell)}$ is $L_r^{(\ell)}$-Lipschitz with $L_r^{(\ell)}=C(\sigma)\,\|\Theta_V^{(\ell)}\|_{\mathrm{op}}\,\|\Theta_Q^{(\ell)}\|_{\mathrm{op}}$, whence
\[
\E\|(\mathrm B)\|^{2}\ \le\ (L_r^{(\ell)})^{2}\,\E\|\epsilon_t^{(\ell-1)}\|^{2}
\ \le\ (L_r^{(\ell)})^{2}\,C_{\ell-1}T_{\ell-1}/t.
\]

\emph{Residual update.}
Lipschitzness of $\Phi^{(\ell)}$ gives
\begin{align*}
\E\|\epsilon_t^{(\ell)}\|^{2}
&\le 2L_\Phi^{2}\Bigl(\E\|\epsilon_t^{(\ell-1)}\|^{2}
+\E\|y_t^{(\ell)}-r^{(\ell)}(F^{(\ell-1)}(x_t))\|^{2}\Bigr)\\
&\le 2L_\Phi^{2}\Bigl(\tfrac{C_{\ell-1}T_{\ell-1}}{t}
+2\E\|(\mathrm A)\|^{2}+2\E\|(\mathrm B)\|^{2}\Bigr)
\ \le\ \frac{C_\ell\,T_\ell}{t},
\end{align*}
with $T_\ell=T_{\ell-1}+\tau_\alpha^{(\ell-1)}\tr(\Sigma_{\ell-1})$ and $C_\ell$ absorbing $L_\Phi$, $L_r^{(\ell)}$, $\|\Theta_V^{(\ell)}\|_{\mathrm{op}}$, and $C_1$. Iterating from $\ell=1$ to $\ell=L$ yields \eqref{eq:multilayer-rate}.
\end{proof}

\subsection{Proof of Proposition~\ref{prop:markov}}
\label{app:markov}

\begin{proof}[Proof of Proposition~\ref{prop:markov}]
We first establish the almost sure convergence of the sampled next token to its context-free limit. Theorem~\ref{thm:multilayer} implies that the terminal hidden state $h_t^{(L)}$ approaches the deterministic readout $F^{(L)}(x_t)$: $\|h_t^{(L)}-F^{(L)}(x_t)\| \to 0$ almost surely as $t\to\infty$. By assumption the decoder $\Psi$ is globally Lipschitz in its first argument, say with constant $L_\Psi$, so for every realisation of the noise $\xi_{t+1}$ we have
\[
 \bigl\|\Psi(h_t^{(L)}, x_t, \xi_{t+1}) - \Psi(F^{(L)}(x_t), x_t, \xi_{t+1})\bigr\| \;\le\; L_\Psi\, \| h_t^{(L)} - F^{(L)}(x_t)\|.
\]
The right-hand side converges to zero almost surely, hence so does the left-hand side. This yields the first displayed convergence in the proposition.

For the conditional-expectation statement fix any bounded Lipschitz function $\varphi:\R^d\to\R$. Define the discrepancy
\[
\Delta_{t+1} \;:=\; \varphi(x_{t+1}) \; - \; \varphi\bigl(\Psi(F^{(L)}(x_t), x_t, \xi_{t+1})\bigr).
\]
By the autoregressive update~\eqref{eq:ar-multi} we have $x_{t+1}=\Psi(h_t^{(L)}, x_t, \xi_{t+1})$, so
\[
\Delta_{t+1} = \varphi\bigl(\Psi(h_t^{(L)}, x_t, \xi_{t+1})\bigr) - \varphi\bigl(\Psi(F^{(L)}(x_t), x_t, \xi_{t+1})\bigr).
\]
Because $\varphi$ is Lipschitz and $\Psi$ is Lipschitz in its first argument, the composition $\varphi\circ \Psi(\cdot, x_t, \xi_{t+1})$ is Lipschitz in its first argument with constant at most $L_\Psi\, \mathrm{Lip}(\varphi)$. Consequently
\[
|\Delta_{t+1}| \;\le\; L_\Psi \, \mathrm{Lip}(\varphi)\, \|h_t^{(L)} - F^{(L)}(x_t)\|,
\]
and the right-hand side converges to zero almost surely by Theorem~\ref{thm:multilayer}. Since $\varphi$ is bounded, $|\Delta_{t+1}|$ is dominated by an integrable random variable and $\xi_{t+1}$ is independent of $\mathcal H_t^{(L)}$ by construction. Dominated convergence therefore gives
\[
 \E[\varphi(x_{t+1}) \mid \mathcal H_t^{(L)}] = \E_\xi[\varphi(\Psi(h_t^{(L)}, x_t, \xi))] = \E_\xi[\varphi(\Psi(F^{(L)}(x_t), x_t, \xi))] + \E_\xi[\Delta_{t+1}],
\]
and the last term $\E_\xi[\Delta_{t+1}]$ tends to zero almost surely. Observing that $\mathcal H_t$ and $\mathcal H_t^{(L)}$ contain the same information up to measurable transformations, and defining
\[
K\varphi(x) := \E_\xi[\varphi(\Psi(F^{(L)}(x), x, \xi))],
\]
we conclude that $\E[\varphi(x_{t+1}) \mid \mathcal H_t] - K\varphi(x_t) \to 0$ almost surely.
\end{proof}

\section{Experiments}
\label{app:exp}


\subsection{Covariance-readout alignment under elliptical inputs}
\label{app:exp-sec323}

\begin{figure}[t]
    \centering
    \begin{subfigure}[t]{0.24\textwidth}
        \centering
        \includegraphics[width=\linewidth]{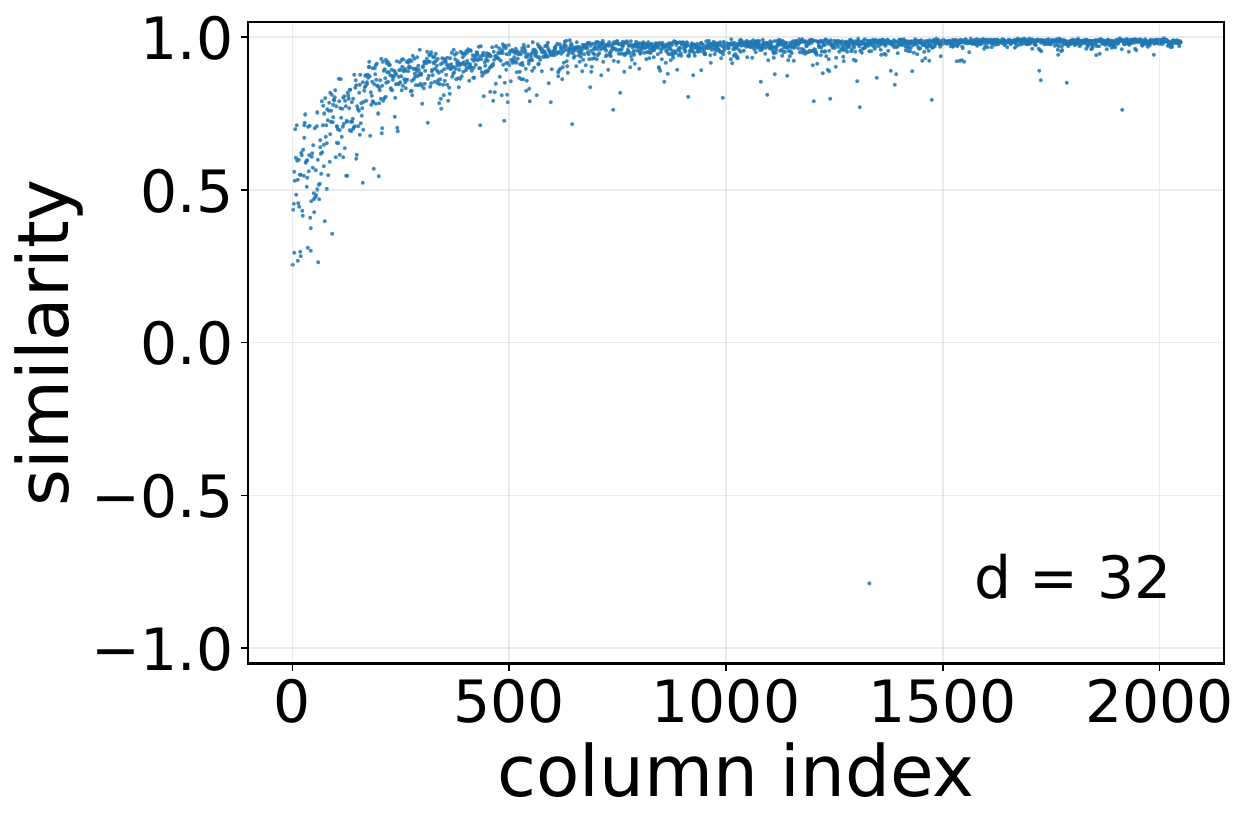}
        \caption{$d=32$}
    \end{subfigure}
    \hfill
    \begin{subfigure}[t]{0.24\textwidth}
        \centering
        \includegraphics[width=\linewidth]{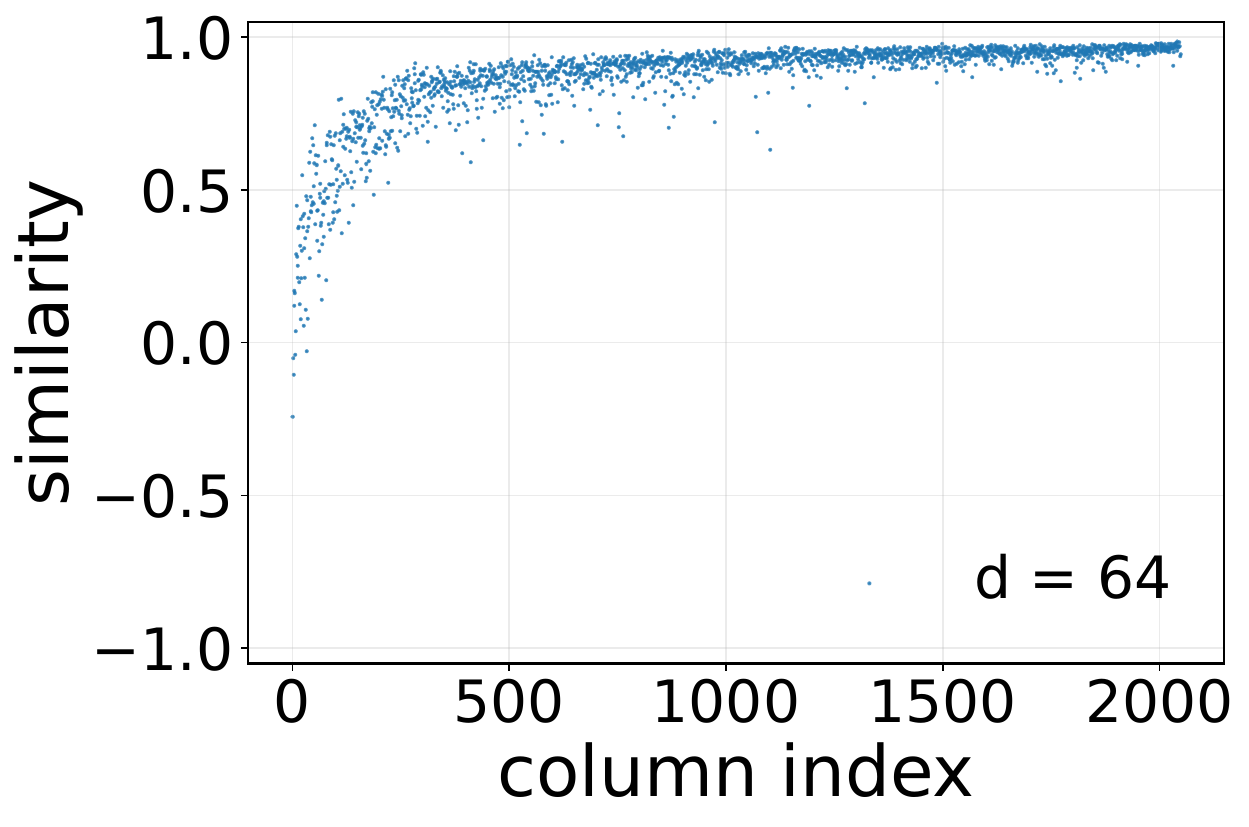}
        \caption{$d=64$}
    \end{subfigure}
    \hfill
    \begin{subfigure}[t]{0.24\textwidth}
        \centering
        \includegraphics[width=\linewidth]{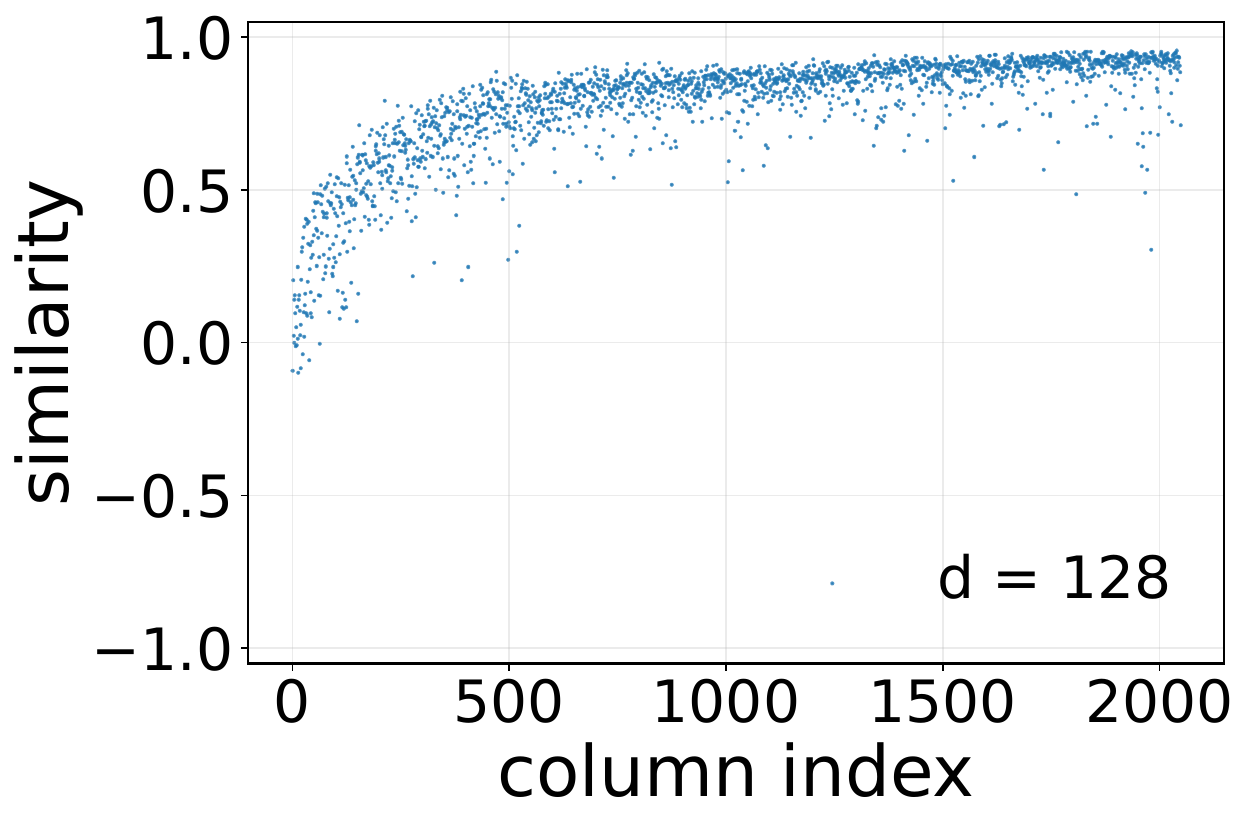}
        \caption{$d=128$}
    \end{subfigure}
    \hfill
    \begin{subfigure}[t]{0.24\textwidth}
        \centering
        \includegraphics[width=\linewidth]{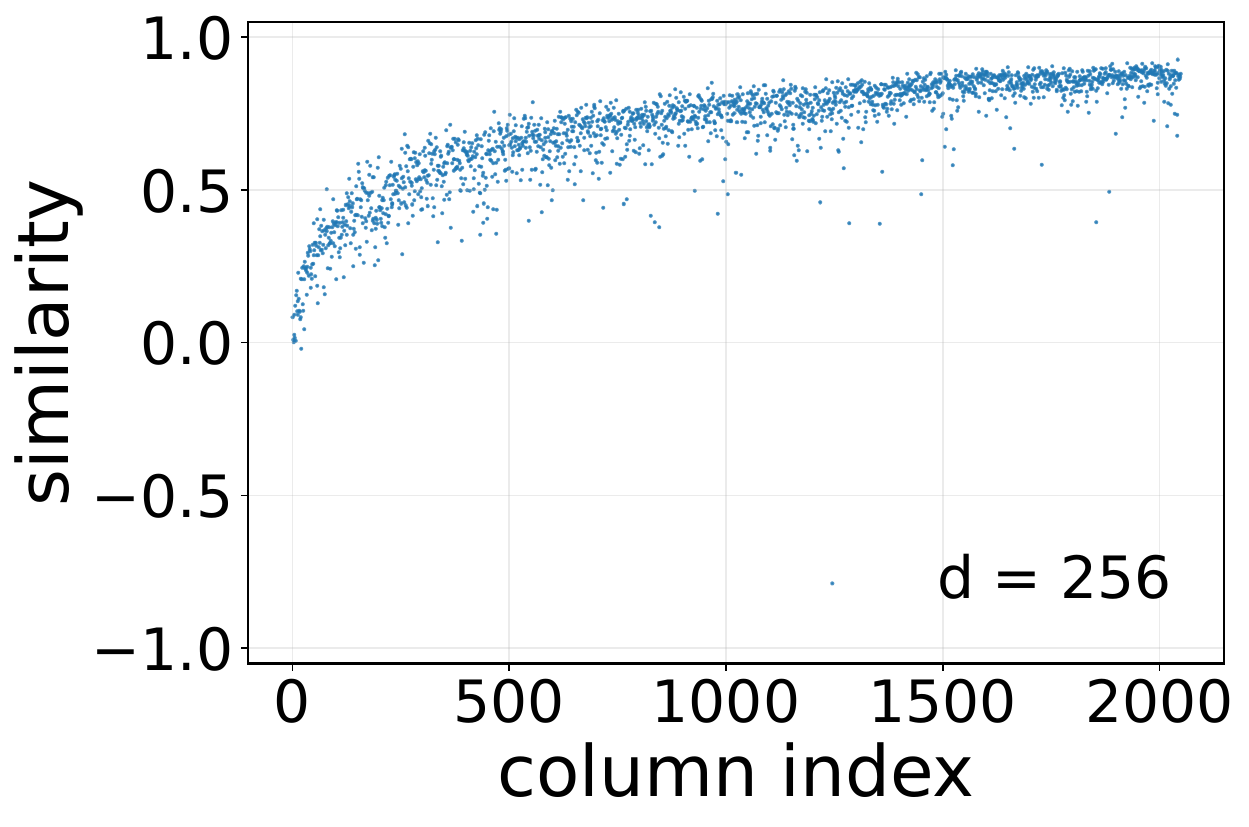}
        \caption{$d=256$}
    \end{subfigure}

    \caption{
    Covariance-readout alignment under elliptical inputs. Each panel shows the cosine similarity between the causal attention output $y_t$ and the Section~\ref{ssec:elliptical} target $\gamma_\Sigma(q_t)\Theta_V\Sigma\Theta_K^\top\Theta_Q x_t$ across token positions. The increasing similarity along the context supports the predicted covariance-rotated alignment.
    }
    \label{fig:sec323-covariance-readout}
\end{figure}

We empirically validate the joint-covariance limit in Section~\ref{ssec:elliptical}. For each dimension $d \in \{32,64,128,256\}$, we sample a positive-definite covariance matrix $\Sigma = BB^\top$ and generate an elliptical input stream $x_t = Bz_t$, where $z_t \sim \mathcal{N}(0,I_d)$. Thus the marginal covariance of the input tokens is exactly $\Sigma$.

For a frozen randomly initialized single-head causal attention layer with projections $\Theta_Q,\Theta_K,\Theta_V$, Section~\ref{ssec:elliptical} predicts that the attention output aligns, in the long-context limit, with the covariance-rotated readout
\[
    \Theta_V \Sigma \Theta_K^\top \Theta_Q x_t ,
\]
up to a positive scalar factor. We report the column-wise cosine similarity across token positions in Figure~\ref{fig:sec323-covariance-readout}.

\subsection{Experimental setup for Section~\ref{sec:icl}}
\label{app:exp-sec4}
This appendix details the synthetic in-context linear-regression experiment whose results are summarised in Figure~\ref{fig:icl-readout}.

\paragraph{Data-generating process.} For each task we draw a diagonal covariance $\Sigma_{uu}\in\R^{d_u\times d_u}$ whose diagonal entries are sampled i.i.d.\ from $\mathrm{Uniform}[0.25,1.0]$, and a regression matrix $\beta\in\R^{d_u\times d_w}$ with i.i.d.\ standard Gaussian entries. Conditional on $(\Sigma_{uu},\beta)$, we generate $t$ i.i.d.\ covariates $u_j\sim\mathcal{N}(0,\Sigma_{uu})$ via a Cholesky factor $\Sigma_{uu}=LL^{\top}$ and set $w_j=\beta^{\top}u_j$ (noise-free regression). Tokens are assembled in the format of Section~\ref{ssec:icl-setup},
\[
x_j=\begin{pmatrix}u_j\\ w_j\end{pmatrix}\in\R^{d_u+d_w}\;\;(j<t),\qquad
x_t=\begin{pmatrix}u_t\\ 0\end{pmatrix},
\]
with $d_u=16$, $d_w=4$. Under this generative model the population blocks of the joint covariance are $\Sigma_{wu}=\beta^{\top}\Sigma_{uu}$ and $B^{\star}=\Sigma_{wu}\Sigma_{uu}^{-1}=\beta^{\top}$, so the Bayes regressor is simply $B^{\star}u_t=\beta^{\top}u_t$.

\paragraph{Population reference predictors.} For each prompt we compute, in closed form on the sampled $(\Sigma_{uu},\beta,u_t)$:
\begin{itemize}
\item the one-step GD target $\Sigma_{wu}u_t=\beta^{\top}\Sigma_{uu}u_t$, used as the reference for the single-head experiment;
\item the $K$-step GD iterate $B^{(K)}u_t$, computed by the recursion $B^{(0)}=0$, $B^{(k+1)}=B^{(k)}-\eta\bigl(B^{(k)}\Sigma_{uu}-\Sigma_{wu}\bigr)$, used as the reference for the stack;
\item the Bayes target $B^{\star}u_t=\beta^{\top}u_t$, computed for completeness.
\end{itemize}

\paragraph{Single-head model.} Implements Corollary~\ref{cor:one-step-bayes} with parameter-frozen weights
\[
\Theta_Q=\Theta_K=S_u\in\R^{d_u\times d},\qquad \Theta_V=S_w\in\R^{d_w\times d}.
\]
Causal attention is computed with the standard $1/\sqrt{d_u}$ scaling in the softmax logits. For Gaussian inputs the scalar $\gamma_\Sigma(q_t)$ of Theorem~\ref{thm:elliptical} equals $1/\sqrt{d_u}$ exactly (since $\nabla\Lambda(u)=\Sigma u$ yields $m_x(q)=\Sigma\Theta_K^{\top}q/\sqrt{d_u}$); we therefore report $\sqrt{d_u}\cdot y_t$ as the head output, which exactly cancels $\gamma_\Sigma(q_t)$ and makes the absolute scale of the prediction comparable to $\Sigma_{wu}u_t$. This affine rescaling does not affect the cosine-similarity metric. Per-prompt-length statistics are averaged over $512$ independent tasks.

\paragraph{Stacked-head model.} Implements the residual construction of~\eqref{eq:weights} with $K=8$ identical blocks. Each block applies the same selectors $\Theta_Q=\Theta_K=S_u$, $\Theta_V=S_w$ to the current hidden state and updates the $w$-slot of every position by
\[
r_j^{(k+1)}\;=\;r_j^{(k)}-\eta\sqrt{d_u}\;y_j^{(k)},
\]
where $y_j^{(k)}$ denotes the standard softmax-attention readout at position $j$ of layer $k$; the $\sqrt{d_u}$ factor again absorbs the $1/\sqrt{d_u}$ scalar $\gamma_\Sigma(q_t)$ so that $\eta$ acts as the effective step size in the population recursion $B^{(k+1)}=B^{(k)}-\eta\,\nabla L(B^{(k)})$. The query prediction is $\hat w_t^{(K)}=-r_t^{(K)}$, in line with Proposition~\ref{prop:depth-gd}. We use $\eta=10^{-2}$, well within the stability range $0<\eta<2/\lambda_{\max}(\Sigma_{uu})$ since our covariance prior enforces $\lambda_{\max}(\Sigma_{uu})\le 1$. All weights are frozen at the prescribed values; no training is performed. Per-prompt-length statistics are averaged over $256$ independent tasks.

\paragraph{Metrics.} For each task and each prompt length we record the cosine similarity $\cos(\hat w_t,\text{target})$ and the squared error $\|\hat w_t-\text{target}\|^2$, and then average across tasks.


\end{document}